\documentclass{article}
\usepackage[final]{colm2026_conference}

\usepackage{microtype}
\usepackage{fontawesome5}
\usepackage[T1]{fontenc}
\usepackage[utf8]{inputenc}

\usepackage{amsmath}
\usepackage{amssymb}
\usepackage{amsthm}
\usepackage{algorithm}
\usepackage{algpseudocode}

\newtheorem{proposition}{Proposition}

\usepackage{lineno}

\usepackage{booktabs}
\usepackage{multirow}
\usepackage{graphicx}
\usepackage{subcaption}
\DeclareCaptionJustification{softraggedright}{\rightskip=0pt plus 2em\relax}
\usepackage{colortbl}
\usepackage{pifont}
\usepackage{enumitem}
\usepackage{placeins}
\usepackage{wrapfig}

\usepackage{tikz}
\usetikzlibrary{arrows.meta, calc, decorations.pathreplacing, backgrounds, shapes.misc, positioning, fit}
\usepackage{hyperref}
\definecolor{darkblue}{rgb}{0, 0, 0.5}
\hypersetup{colorlinks=true, citecolor=darkblue, linkcolor=darkblue, urlcolor=darkblue}
\usepackage{cleveref}

\newcommand{\method}{\textsc{Goose}}
\newcommand{\gooseone}[4]{\begin{scope}[shift={(#1,#2)},scale=#3]
  \fill[#4]
    (-1.05,0.00)
    .. controls (-0.55,-0.14) and (0.05,-0.16) .. (0.42,-0.06)
    .. controls (0.62,0.00) and (0.80,0.10) .. (0.98,0.24)
    .. controls (1.06,0.30) and (1.16,0.28) .. (1.22,0.24)
    .. controls (1.14,0.36) and (1.02,0.40) .. (0.92,0.36)
    .. controls (0.70,0.28) and (0.52,0.18) .. (0.30,0.14)
    .. controls (0.05,0.55) and (-0.25,0.85) .. (-0.42,0.95)
    .. controls (-0.40,0.70) and (-0.45,0.40) .. (-0.60,0.16)
    .. controls (-0.80,0.10) and (-0.95,0.06) .. (-1.05,0.00) -- cycle;
  \fill[#4]
    (0.15,-0.02) .. controls (-0.05,-0.30) and (-0.25,-0.42) .. (-0.42,-0.46)
    .. controls (-0.28,-0.28) and (-0.12,-0.10) .. (0.02,0.02) -- cycle;
\end{scope}}
\newcommand{\gooseflock}{\raisebox{1.00ex}{\begin{tikzpicture}[scale=0.27,baseline=0pt]
  \gooseone{0}{0}{1}{pldblue}            % lead   (spine)
  \gooseone{-2.3}{0.56}{0.88}{trorange}  % upper trailing (branch)
  \gooseone{-2.3}{-0.56}{0.88}{trorange} % lower trailing (branch)
\end{tikzpicture}}}
\newcommand{\taumetric}{\ensuremath{\tau}}
\newcommand{\argmax}{\operatorname*{arg\,max}}
\newcommand{\pld}{\textsc{PLD}}
\newcommand{\tr}{\textsc{TR}}
\definecolor{bestcolor}{HTML}{E8F5E9}
\definecolor{secondcolor}{HTML}{FFF3E0}
\definecolor{pldblue}{RGB}{21,101,192}
\definecolor{trorange}{RGB}{230,81,0}
\definecolor{accgreen}{RGB}{46,125,50}
\definecolor{rejred}{RGB}{198,40,40}

\title{\gooseflock\,\ Goose: Anisotropic Speculation Trees \\ for Training-Free Speculative Decoding}

\author{Tao Jin\thanks{Corresponding author.} \quad Phuong Minh Nguyen \quad Naoya Inoue \\
Japan Advanced Institute of Science and Technology (JAIST) \\
\texttt{\{morgan, phuongnm, naoya-i\}@jaist.ac.jp} \\
\faGithub~\href{https://github.com/2kinto/GOOSE_Speculative_Decoding}{\texttt{GOOSE\_Speculative\_Decoding}}}

\begin{document}

\ifcolmsubmission
\linenumbers
\fi

\maketitle

% ═══════════════════════════════════════════════════════════════════════
% ABSTRACT
% ═══════════════════════════════════════════════════════════════════════
\begin{abstract}
Speculative decoding accelerates large language model inference by drafting multiple candidate tokens and verifying them in a single forward pass.
Candidates are organized as a tree: deeper trees accept more tokens per step, but adding depth requires sacrificing breadth (fallback options) under a fixed verification budget.
Existing training-free methods draft from a single token source and shape their trees without distinguishing candidate quality across origins.
We observe that two common training-free token sources---$n$-gram matches copied from the input context, and statistical predictions from prior forward passes---differ sharply in acceptance rate (${\sim}6\times$ median gap, range $2$--$18\times$ across five models and five benchmarks).
We prove that when such a quality gap exists, the optimal tree is \emph{anisotropic} (asymmetric): reliable tokens should form a deep chain while unreliable tokens spread as wide branches, raising the depth ceiling of balanced trees.
We realize this structure in \method{}, a training-free framework that builds an adaptive \textbf{spine tree}---a deep chain of high-acceptance context-matched tokens with wide branches of low-acceptance alternatives at each node.
The resulting tree provably accepts at least as many tokens per step as either source alone.
On five LLMs (7B--33B) and five benchmarks, \method{} achieves $1.9$--$4.3\times$ lossless speedup, outperforming balanced-tree baselines by $12$--$33\%$ under the same budget.
\end{abstract}

% ═══════════════════════════════════════════════════════════════════════
% 1. INTRODUCTION
% ═══════════════════════════════════════════════════════════════════════
\section{Introduction}
\label{sec:intro}

Speculative decoding~\citep{leviathan2023fast,chen2023accelerating} accelerates autoregressive LLM inference by letting a drafter propose several candidate tokens at once.
The target model verifies them in a single forward pass, accepting correct guesses at no extra cost while keeping the output provably identical to standard decoding (\emph{lossless};~\citealp{leviathan2023fast}).

The simplest draft is a single \emph{chain}: efficient when the drafter is accurate, but a single rejection loses the entire continuation.
\emph{Tree-structured speculation}~\citep{miao2024specinfer,cai2024medusa} hedges this risk with multiple candidates organized as a tree, whose \textbf{tree attention mask} restricts each candidate's attention to its ancestors so the target model scores every root-to-leaf path in one forward pass.
Designing a good \emph{draft tree} then means allocating a fixed node budget~$B$ between \emph{depth} (longer paths) and \emph{breadth} (more fallbacks).

Among training-free methods, three representative approaches illustrate the design space and its limits.
\textit{Token Recycling}~(\tr{};~\citealp{luo2025tokenrecycling}) retrieves statistically likely next tokens from the target model's own logits, producing wide, shallow trees that hedge broadly but sacrifice depth.
\textit{Prompt Lookup Decoding}~(\pld{};~\citealp{saxena2023prompt}) copies $n$-gram-matched continuations from the input context, producing a deep chain accepted at high rates, but with no fallback when the match breaks or is absent.
Sequoia~\citep{chen2024sequoia} optimizes the depth--breadth trade-off by dynamic programming, but under a source-blind acceptance model: acceptance depends on a candidate's rank among siblings, not on its provenance.
Such a model cannot favor high-acceptance tokens when sources of different reliability coexist.

This source-blind assumption forfeits a key structural advantage once both sources coexist.
We measure that context-matched (\pld{}) tokens are accepted $2$--$18\times$ more often than transition (\tr{}) tokens (median~${\sim}6\times$ across five models and five benchmarks; \Cref{sec:theory}).
This \emph{acceptance heterogeneity} suggests placing high-acceptance tokens along a deep \emph{spine} and low-acceptance tokens as wide \emph{branches} at every spine node: the spine pushes depth while branches provide fallback when it breaks early.
We instantiate this insight in \method{}, a training-free framework that builds an adaptive \textbf{spine tree} (\Cref{fig:overview}): a deep chain of context-matched tokens, with wide branches of transition alternatives forking from every spine node, verified by a unified greedy walk in one forward pass.
We prove that the spine tree's expected yield (accepted tokens per cycle) is at least that of either source used alone (\Cref{prop:monotonicity}), and design runtime adaptations that preserve this property when one source is absent or dominant (\Cref{sec:consensus}).

Our contributions:
\begin{enumerate}[nosep,leftmargin=*]
\item We formalize acceptance heterogeneity between two training-free token sources and prove that the optimal tree is \emph{anisotropic}, allocating depth to reliable tokens and breadth to unreliable ones, which raises the depth ceiling of single-source trees (\Cref{sec:theory}).
\item We design a confidence-adaptive, \textbf{training-free} tree construction that combines \pld{} and \tr{} into a single spine tree with a non-degradation guarantee on expected yield (\Cref{sec:method}).
\item Experiments on five models (7B--33B) across five benchmarks show $1.9$--$4.3\times$ lossless speedup, outperforming isotropic-tree baselines by $12$--$33\%$ under equal node budgets (\Cref{tab:topology}).
\end{enumerate}

% ═══════════════════════════════════════════════════════════════════════
% 2. RELATED WORK
% ═══════════════════════════════════════════════════════════════════════
\section{Related Work}
\label{sec:related}

Speculative decoding~\citep{leviathan2023fast,chen2023accelerating,stern2018blockwise} accelerates autoregressive LLM inference by guessing multiple future tokens and verifying them in a single forward pass; \citet{xia2024unlocking} provide a comprehensive survey.
A key extension is \emph{tree-structured verification}~\citep{miao2024specinfer,cai2024medusa}: candidates are organized into a tree and checked simultaneously, after which the longest accepted path is selected and one bonus token is appended.

\paragraph{Draft-model methods.}
Representative approaches include EAGLE~\citep{li2024eagle,li2024eagle2}, which trains feature-level draft heads, and Medusa~\citep{cai2024medusa}, which attaches parallel prediction heads; EAGLE-3~\citep{li2025eagle3} replaces feature prediction with direct token prediction and multi-layer feature fusion, achieving up to ${\sim}6.5\times$ speedup.
All such methods~\citep[\emph{inter alia}]{miao2024specinfer,kim2023speculative,zhou2024distillspec} require per-model training or auxiliary models, limiting deployment flexibility.

\paragraph{Training-free methods.}
Training-free drafters avoid auxiliary models entirely by mining candidate tokens from data already available at inference time.
The simplest strategy copies tokens from reference documents~\citep{yang2023inference} or from the prompt and prior output via $n$-gram matching~\citep{saxena2023prompt}; such copies are cheap but limited to recurring patterns.
A broader class retrieves candidates from external datastores---REST~\citep{he2024rest} uses a suffix array, SuffixDecoding~\citep{oliaro2024suffix} a suffix tree over the prompt and prior outputs (compared head-to-head in \Cref{sec:suffix}), while Lookahead Decoding~\citep{fu2024break} constructs $n$-grams on the fly through parallel Jacobi iterations.
Other methods refine token selection~\citep{somasundaram2024pldplus,le2025spectra} or reuse the model's own early layers~\citep{zhang2024draft,elhoushi2024layerskip}.
All the above improve \textit{which tokens to draft}.
\method{} asks the orthogonal question of \textit{how the verification tree should be shaped} given sources with very different acceptance rates; better token selection would only widen the gap the spine tree exploits.

\paragraph{Tree topology optimization.}
Sequoia~\citep{chen2024sequoia} formalizes the depth--breadth trade-off introduced above: under its \emph{positional acceptance assumption}, acceptance depends only on a token's rank among siblings, so its DP produces non-uniform trees while remaining \emph{source}-blind---one acceptance vector for all candidates.
We call a balanced tree with the same branching at every depth \emph{isotropic} and use it as our equal-budget topology control (\Cref{sec:topology}); we additionally compare against Sequoia's official DP-optimized trees in \Cref{sec:sequoia-headtohead}.

EAGLE-2~\citep{li2024eagle2} prunes low-confidence branches via the draft head's per-token scores, growing deeper along well-predicted paths; this adapts the shape to within-drafter quality variation but requires a trained head and lacks optimality guarantees.
\tr{}~\citep{luo2025tokenrecycling} applies a similar principle within its transition lookup table: higher-ranked candidates receive more children than lower-ranked ones, producing non-uniform trees without training.

Despite producing non-uniform trees, these methods optimize within a single draft source, so the quality differences they exploit remain \emph{within-source}.
\method{} exploits a qualitatively different axis: the \emph{cross-source} gap between context-matched and transition tokens ($2$--$18\times$; \Cref{sec:theory}).
Extending Sequoia's DP to this heterogeneous setting yields an \emph{anisotropic} optimum, in which reliable tokens form a deep chain while unreliable ones spread as wide branches, confirmed by 12--33\% gains (\Cref{tab:topology}).

% ═══════════════════════════════════════════════════════════════════════
% THEORETICAL ANALYSIS
% ═══════════════════════════════════════════════════════════════════════
\section{Optimal Topology under Heterogeneous Acceptance}
\label{sec:theory}

The optimal tree shape follows directly from the acceptance asymmetry between the two draft sources.
We build on Sequoia's framework~\citep{chen2024sequoia}, which models each draft token as independently accepted with a source-independent probability and optimizes the tree via dynamic programming; we extend the model to two sources with different acceptance rates.
This section formalizes that asymmetry, derives the expected yield of the spine tree (\Cref{prop:yield}), proves that branches should be concentrated near the root (\Cref{prop:allocation}), and shows that the resulting tree strictly dominates any single-source alternative (\Cref{prop:dominance}).

\paragraph{Heterogeneous acceptance model.}
Following this model, we associate each draft source with a per-token acceptance probability:
$p_s$ for context-matched (\emph{s}pine) tokens and $p_t$ for \emph{t}ransition tokens, with $p_s > p_t$.
Each spine token is modeled as independently accepted with probability~$p_s$.
Consecutive tokens are positively correlated in practice; this leaves the bound in \Cref{eq:yield} valid, only more conservative (validated in all 24 settings, \Cref{app:eq5-verify}).

\begin{figure*}[!t]
\noindent
\begin{minipage}[t]{0.487\textwidth}
\vspace{0pt}
\centering
\includegraphics[width=1.00\textwidth]{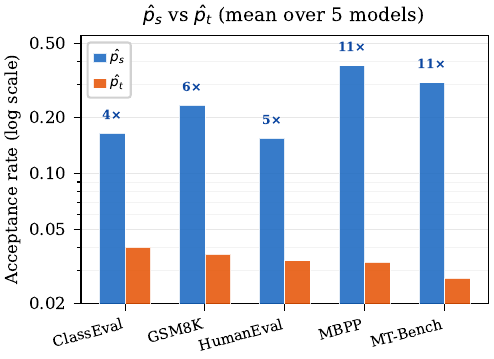}
\captionsetup{justification=softraggedright,singlelinecheck=false}
\vspace{-2pt}
\captionof{figure}{\textbf{Acceptance heterogeneity: context-matched spine tokens are accepted far more often than transition branch tokens.}
\mbox{$\hat{p}_s$ / $\hat{p}_t$} = fraction of \pld{} spine / \tr{} branch draft tokens accepted (greedy verification logs); bars: means over the five models, log scale; $N\times$: the ratio, which exceeds $1$ in all $25$ settings---the empirical basis for the anisotropic design.}
\label{fig:heterogeneity}
\end{minipage}%
\hfill
\begin{minipage}[t]{0.435\textwidth}
\vspace{0pt}
\centering
\includegraphics[width=0.927\textwidth]{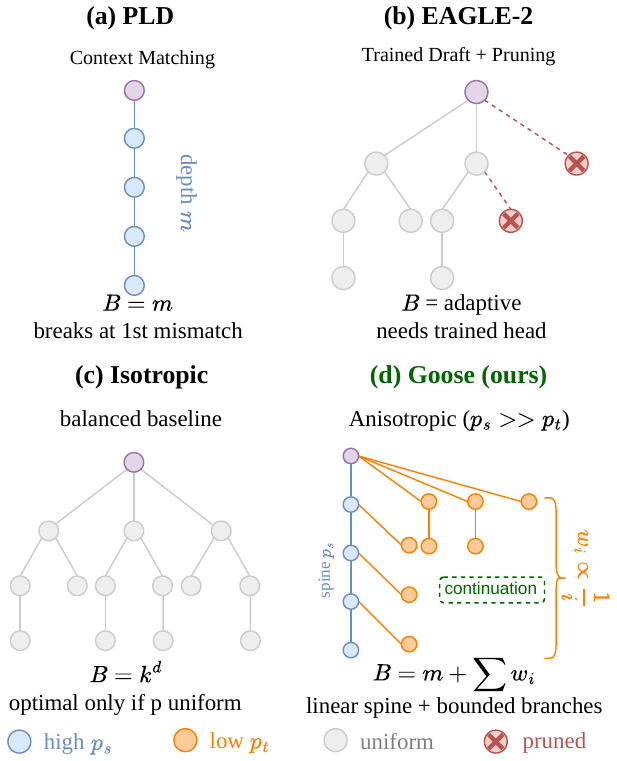}
\captionsetup{justification=softraggedright,singlelinecheck=false}
\vspace{-4pt}
\captionof{figure}{\textbf{Speculation-tree topologies} ($B$ fixed).
\textbf{(a)}~\pld{} spine; \textbf{(b)}~EAGLE-2 pruned tree;
\textbf{(c)}~isotropic tree, our equal-budget control;
\textbf{(d)}~\method{}'s anisotropic spine tree (\Cref{sec:tree}).}
\label{fig:overview}
\end{minipage}
\vspace{-16pt}
\end{figure*}

To ground the model, we measure acceptance rates from greedy verification logs across all five models and five benchmarks (\Cref{fig:heterogeneity}).
The empirical spine acceptance~$\hat{p}_s$ (fraction of context-matched draft tokens accepted by the target model) ranges from $0.07$ to $0.60$ (median~$0.21$), while the empirical transition acceptance~$\hat{p}_t$ lies in $0.03\text{--}0.05$ (median~$0.033$).
The resulting heterogeneity ratio $\hat{p}_s/\hat{p}_t$ spans $2\times$ to $18\times$ (median~$\approx 6\times$; \Cref{fig:heterogeneity}).
Note that $p_s$ and $p_t$ are \emph{population} rates, used in this yield analysis and in tree construction only: verification (\Cref{sec:continuation}) is pure argmax token-identity matching.
When \Cref{prop:dominance} ``upgrades'' a slot from $p_t$ to $p_s$, this describes which source class occupies it in the model, not a runtime score; a branch continuing after a spine break keeps its own class rate $p_t$ and inherits nothing from the spine.

\paragraph{Key quantities.}
Before stating the main results, we fix notation.
The spine consists of $m$ context-matched tokens forming a linear chain; at each spine node~$i$ ($0 \le i \le m{-}1$), $w_i$ transition candidates fork off as branches.
The total number of candidates---spine plus all branches---is bounded by the node budget~$B$ (i.e., $m + \sum_i w_i \le B$; branch extension chains of up to $D{-}1$ nodes are idealized as budget-free here, while the deployed builder charges every node to $B$, \Cref{sec:tree}).
Let $\tau$ denote the number of tokens accepted in one verification cycle~\citep{leviathan2023fast,chen2024sequoia}, which we call the \emph{compression ratio}; the \emph{expected yield} is then $\mathbb{E}[\tau]$.
We also define two recurring terms.
Given $w_i$ branch candidates at spine node~$i$, each accepted independently with probability~$p_t$, the probability that \emph{at least one} is correct is
$\phi_i \;=\; 1 - (1-p_t)^{w_i}$.
Intuitively, $\phi_i$ grows quickly with $w_i$ even when $p_t$ is small.
Each accepted branch token can extend into a chain of further transition tokens (maximum depth $D{=}6$ in all experiments), reaching depth~$k$ with probability~$p_t^k$; summing gives the expected extension length:
$\bar\ell \;=\; \sum_{k=1}^{D-1} p_t^k \;=\; p_t(1-p_t^{D-1})/(1-p_t)$,
which ranges from~$0$ (when~$p_t \to 0$) to~$D{-}1$ (when~$p_t \to 1$).

We now bound $\mathbb{E}[\tau]$ via the \emph{indicator random variable} (IRV) decomposition~\citep{leviathan2023fast,chen2024sequoia}: the expected accepted token count equals the sum over all tree depths of the probability that the accepted path reaches and accepts that depth.

\begin{proposition}[Spine Tree Expected Yield]
\label{prop:yield}
Under the heterogeneous model with spine length $m \geq 1$ and branch widths $\{w_i\}_{i=0}^{m-1}$ (total budget $m + \sum_i w_i \leq B$), the expected accepted path length satisfies:
\begin{equation}
\label{eq:yield}
\mathbb{E}[\tau] \;\geq\; \underbrace{\textstyle\sum_{i=1}^{m} p_s^i}_{\textup{spine}} \;+\; \underbrace{\textstyle\sum_{i=0}^{m-1} p_s^{i}(1{-}p_s)\,\phi_i\,(1{+}\bar\ell)}_{\textup{continuation synergy}} \;+\; 1
\end{equation}
with $\phi_i$ and $\bar\ell$ as defined above.
The bound is tight when the branches are independent chains.
\end{proposition}

\noindent\textit{Proof sketch.}
By the IRV decomposition~\citep{leviathan2023fast,chen2024sequoia}, $\mathbb{E}[\tau]$ is a sum over depths of the probability the accepted path reaches that depth: the spine term is $\sum_i p_s^i$; in the synergy term, $p_s^{i}(1{-}p_s)$ is the probability the spine breaks exactly at depth $i{+}1$, $\phi_i$ the probability some branch there succeeds, and $(1{+}\bar\ell)$ its expected token contribution.
The final $+1$ is the standard bonus token; deeper sub-branching only adds contributions, so Eq.~\ref{eq:yield} is a lower bound.
Full derivation in \Cref{app:yield-proof}. \hfill$\square$

Since the spine tree includes both the pure \pld{} chain ($m{=}B$, all~$w_i{=}0$) and a near-pure \tr{} tree ($m{=}1$) as special cases, its yield never falls below either standalone strategy; $p_s{\to}1$ recovers pure spine, $p_s{\to}p_t$ recovers isotropic.
In practice, the spine ratio adapts via an exponential moving average of the acceptance rate (\Cref{sec:consensus}).

\begin{proposition}[Optimal Branch Allocation]
\label{prop:allocation}
The allocation $\{w_i^*\}$ maximizing the synergy term under total branch budget $B_t = B - m$ satisfies:
\begin{equation}\label{eq:alloc}
w_i^* \;=\; w_0^* \;-\; \frac{|\ln p_s|}{|\ln(1{-}p_t)|} \cdot i
\end{equation}
where $w_0^*$ is set by the budget constraint $\sum w_i = B_t$ (interior solution; otherwise its water-filling truncation, \Cref{app:powerlaw-allocation}).
The allocation decreases linearly with depth---deeper nodes are reached less often ($p_s^i$ decays exponentially)---and the slope $|\ln p_s|/|\ln(1{-}p_t)|$ equalizes the marginal value of one extra branch across positions.
\end{proposition}

\noindent The proof follows from Lagrange multipliers; see \Cref{app:allocation-proof}.
In practice, we approximate the linear schedule with a $1/i$ harmonic rule (\Cref{alg:build-spine}) that preserves the monotone-decreasing shape without estimating acceptance rates. Empirically, $\tau$ is robust to the $r$/$\rho$ split within this family (\Cref{fig:heterogeneity-decomp}b); a measured-acceptance allocator improves on the harmonic rule by a further ${\sim}6\%$ (\Cref{sec:allocator}).

\paragraph{Remark (i.i.d.\ as a modeling approximation).}
The i.i.d.\ acceptance model in Propositions~\ref{prop:yield}--\ref{prop:dominance} is a \emph{modeling approximation}, and it enters only the theoretical analysis, not the measurements: the headline $12$--$33\%$ anisotropic-vs-isotropic gain (\Cref{tab:topology}) is a direct equal-budget measurement that assumes no acceptance law.
Positive intra-chain correlation, if anything, makes the modeled allocation conservative; empirical gains in \Cref{tab:topology} confirm the model captures the dominant qualitative structure (\Cref{app:eq5-verify}).
Empirically, acceptance decay over depth is fit better by a geometric law ($R^2{=}0.77$ vs.\ $0.69$; Qwen3-8B traces, ${\approx}0.5$M tree nodes per depth), whereas the all-reject frequency as a function of branch width is fit better by a power law (Llama-3 traces, \Cref{app:r_k-empirical}; with a selection-effect caveat).
In \Cref{app:powerlaw-allocation} we re-derive the optimal allocation under that empirically-fit tail ($\widehat{\alpha}{=}0.261$) and show the optimum keeps the same root-concentrated anisotropic shape; our deployed harmonic schedule is more conservative than either optimum at the root while retaining residual depth-side breadth.

\begin{proposition}[Spine Tree Dominance]
\label{prop:dominance}
For acceptance rates $p_s > p_t > 0$ and any budget $B$ large enough that the spine contains at least two nodes ($B \geq m+2$, practically $B \geq 10$ for $p_s/p_t \geq 8$), the spine tree with optimally allocated branches (Eq.~\ref{eq:alloc}) achieves strictly higher expected yield than the best single-source isotropic tree of the same budget:
\begin{equation}
\mathbb{E}[\tau_{\textup{spine}}^*] \;>\; \mathbb{E}[\tau_{\textup{iso}}^*] \quad \text{whenever } p_s > p_t.
\end{equation}
The gap increases monotonically with $ p_s/p_t$ and vanishes as $p_s \to p_t$.
\end{proposition}

\noindent\textit{Proof sketch.}
Start from the optimal isotropic tree of budget~$B$ (the single-rate optimum of Sequoia-style DP~\citep{chen2024sequoia} when every token carries rate~$p_t$).
Upgrade one token per depth from~$p_t$ to~$p_s > p_t$, creating a spine.
At each modified depth~$d$, the probability of at least one accepted candidate increases by $(p_s{-}p_t)(1{-}p_t)^{w_d-1} > 0$, a strictly positive improvement that propagates through the yield.
Optimizing the allocation (\Cref{prop:allocation}) adds further gains.
When $p_s \to p_t$, the upgrade vanishes, recovering the isotropic baseline.
Full proof in \Cref{app:dominance-proof}. \hfill$\square$

% ═══════════════════════════════════════════════════════════════════════
% METHOD
% ═══════════════════════════════════════════════════════════════════════
\section{Method: \method{}}

\begin{figure*}[!t]
\centering
\includegraphics[width=0.97\textwidth]{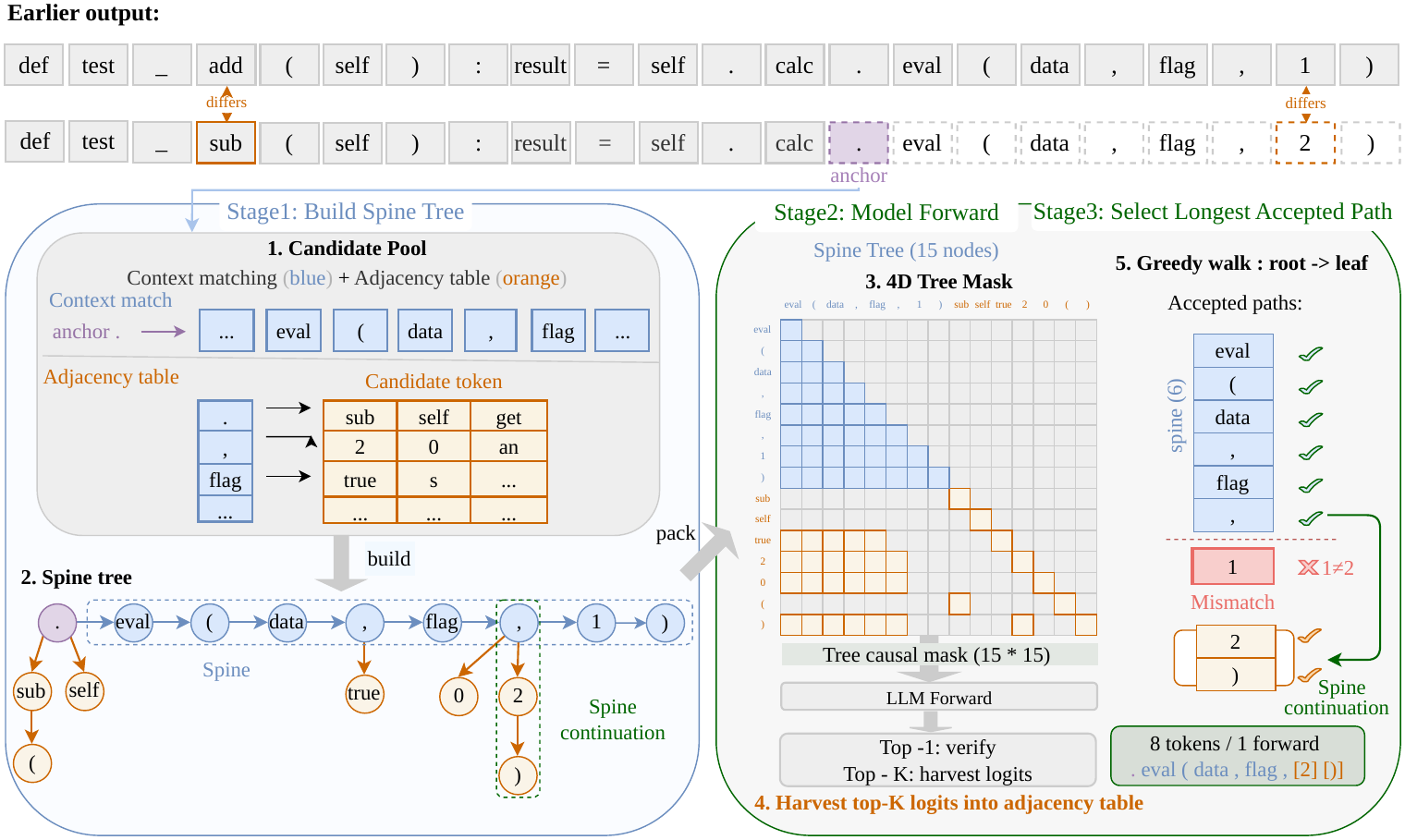}
\vspace{-4pt}
\caption{\textbf{\method{} pipeline overview.}
From anchor~\texttt{.}, Stage~1 draws from a unified
\emph{candidate pool}: context matching produces the spine
(\textcolor{pldblue}{blue}), the adjacency table supplies branches
(\textcolor{trorange}{orange}).
Stage~2 verifies all candidates via one LLM forward pass.
Stage~3 selects the longest accepted path
via a greedy walk, discovering \textbf{spine
continuation}: ``\texttt{2}''$\,\to\,$``\texttt{)}'' extends beyond the spine
mismatch, yielding 8~tokens per call.}
\label{fig:pipeline}
\vspace{-16pt}
\end{figure*}

\label{sec:method}

% ═══════════════════════════════════════════════════════════════════════
% Figure: Pipeline Overview (full-width) + Tree Topology Comparison
% ═══════════════════════════════════════════════════════════════════════

\method{} follows the standard speculative-decoding loop~\citep{cai2024medusa}---generate candidates, verify them in one forward pass, accept the longest correct path---but reshapes the candidate tree into a single asymmetric structure that combines two complementary token sources (\Cref{fig:overview}d).
Each decoding cycle begins from an \emph{anchor token}, the last token accepted in the previous cycle, and proceeds in three stages (\Cref{fig:pipeline}).
\textit{Stage~1---Build spine tree.}
\method{} assembles a \emph{candidate pool} from two sources.
Context matching (\pld{}) retrieves a continuation from previously generated text, forming a deep chain called the \emph{spine}.
The \emph{adjacency table}, our bigram extension of \tr{}'s adjacency matrix (\Cref{sec:tree}; implementation in \Cref{app:impl}), supplies statistically likely alternatives that fork from spine nodes as \emph{branches}.
\textsc{BuildSpineTree} allocates the node budget~$B$ between spine and branches (\Cref{sec:tree}), then packs all candidates into one sequence with a \emph{tree attention mask}: a modified causal mask that lets each node attend only to its root-to-node ancestors, so all paths are verified in one forward pass.
\textit{Stage~2---Verify and harvest.}
One LLM forward pass scores every candidate.
All logits---including those at rejected branches---are harvested into the adjacency table, and the accepted prefix updates the \pld{} context index, enriching both sources for future cycles.
\textit{Stage~3---Select longest correct path.}
A greedy walk identifies the longest path verified as correct by the target model (\Cref{sec:continuation}); the last accepted token becomes the anchor of the next cycle.

The tree shape is not fixed.
When context-match confidence is high, \method{} skips tree construction and verifies the \pld{} chain linearly (\emph{bypass mode}; \Cref{sec:consensus}); when no context match exists, the spine vanishes and the full budget forms a \tr{}-only tree.
The asymmetric topology enables a synergy we call \emph{spine continuation}: when a mismatch occurs on the spine, a branch at that point can extend the accepted path, recovering tokens that neither source alone would have found.

Two refinements improve candidate quality (ablated in \Cref{tab:ablation}): agreement among $n$-gram lengths on the first token is treated as high confidence and extends the spine aggressively, and the adjacency table uses bigram rather than unigram conditioning to sharpen branch predictions.
The complete decode loop is in \Cref{alg:goose} (Appendix).

\subsection{Spine Tree Construction}
\label{sec:tree}

\paragraph{Tree construction.}
Given the anchor token and node budget~$B$, one parameter controls the budget split: the \textit{spine ratio}~$r$, which sets what fraction of~$B$ goes to spine nodes (adapted online; \Cref{sec:consensus}).
Construction proceeds in four steps (\Cref{alg:build-spine}, Appendix):
(1)~lay the spine---allocate up to $B {\cdot} r$ nodes as a linear chain of context-matched tokens;
(2)~add root branches---attach the top adjacency-table alternatives at the root, using a share $(1{-}\rho)$ of the remaining budget (spine-branch ratio $\rho{=}0.5$ by default);
(3)~add spine branches---distribute the remaining share~$\rho$ across spine nodes, with wider branches near the root and narrower ones deeper, approximating the optimal allocation derived in \Cref{prop:allocation};
(4)~extend all branches recursively via the adjacency table up to a maximum depth of~6.
As introduced above, the adjacency table conditions on the two most recent tokens rather than \tr{}'s original single-token lookup~\citep{luo2025tokenrecycling}, improving prediction accuracy.

% ═══════════════════════════════════════════════════════════════════════
% Figure 1: Tree Topology Comparison (single-column, 2×2)
% (Tree comparison figure moved into combined figure* above)

\subsection{Unified Verification and Non-Degradation}
\label{sec:continuation}

The spine tree implicitly contains all candidate paths; verification requires only a single greedy walk.

\paragraph{Greedy walk with source priority.}
Starting at the root, the walk advances to a child whose draft token equals the target model's greedy argmax at the current node---pure token-identity matching, with \emph{no per-token score gate}: draft-side scores (\tr{} adjacency counts, \pld{} match length) are used only at construction time and are never consulted during verification, and no acceptance probability is carried from parent to child.
At each node the walk checks the \pld{} (spine) child before the \tr{} (branch) children; duplicates are excluded, so at most one child can match and the priority never changes the accepted path; once the walk leaves the spine for a branch it does not return (\Cref{app:verify}).
The walk terminates when no child matches, and the accepted path falls into one of three categories:
(a)~\textit{Pure \pld{}} (all spine nodes, equivalent to standalone context-match),
(b)~\textit{Spine continuation} (spine followed by a \tr{} branch at the mismatch point---the synergy unique to the spine tree), or
(c)~\textit{Pure \tr{}} (root branch only, equivalent to standalone transition).
No multi-path bookkeeping is needed; the greedy walk automatically discovers the longest accepted path.
At termination, the target model samples one additional \emph{bonus token} from its distribution at the last accepted position, following the standard speculative-decoding guarantee~\citep{leviathan2023fast}.

\begin{proposition}[Non-Degradation Guarantee]
\label{prop:monotonicity}
For any node budget~$B$, the spine tree's \emph{expected yield}---the expected number of tokens accepted per cycle---is at least as large as that of standalone context matching or standalone transition-tree decoding, each given the full budget~$B$.
\end{proposition}
\noindent\textit{Proof sketch.}
(i)~Standalone \pld{} yields a chain bounded by the match length, not~$B$; \method{} preserves those tokens in full (short matches fit within $\lfloor B r \rfloor$; long matches trigger bypass, \Cref{sec:consensus}) and fills the unused budget with branches, so its accepted path is at least as long.
(ii)~Versus a standalone transition tree, the spine swaps breadth nodes for higher-acceptance context-matched depth, which can only maintain or increase expected path length (\Cref{prop:dominance}); with no context match, the spine vanishes and the transition tree is recovered exactly. \hfill$\square$

\noindent\textit{Remark.}
The guarantee holds for expected yield, not directly wall-clock, and parts~(i)/(ii) rely on the runtime adaptations of \Cref{sec:consensus} (minimum spine ratio, bypass, graceful fallback) for the boundary cases a fixed shape cannot cover; \Cref{tab:ablation} quantifies their contributions.

The greedy walk is formalized in \Cref{alg:verify}; the full decoding loop, including tree construction and logit harvesting, is in \Cref{alg:goose} (Appendix).

\subsection{Confidence-Adaptive Topology}
\label{sec:consensus}

Not every cycle benefits equally from tree construction.
\method{} uses a continuous \textit{PLD confidence} signal to modulate tree shape, ranging from a full spine tree down to a simple linear chain.

\paragraph{Confidence signal.}
\method{} queries context matching with multiple $n$-gram lengths $\{3,4,5\}$.
Confidence is a function of two indicators: (i)~\textit{consensus}---at least two $n$-gram lengths return continuations that agree on the first token, and (ii)~\textit{chain length}---long matches (8 or more tokens) indicate high PLD acceptance.
When confidence is high, tree construction is skipped and the spine is verified as a single linear sequence (\emph{bypass} mode), preserving the guarantee of \Cref{prop:monotonicity}.
The spine ratio~$r$ is adjusted each cycle by an EMA of the PLD acceptance rate: high averages deepen the spine ($r \to 0.50$), low ones shift budget to branches ($r \to 0.15$), converging within 3--5 cycles (details in \Cref{app:hyperparams}).

% ═══════════════════════════════════════════════════════════════════════
% EXPERIMENTS
% ═══════════════════════════════════════════════════════════════════════
\section{Experiments}
\label{sec:experiments}

\subsection{Setup}

\paragraph{Models, hardware, and datasets.}
We evaluate five instruction-tuned models spanning 7B--33B parameters: Vicuna-7B/13B/33B~\citep{zheng2024judging}, Llama-3-8B-Instruct~\citep{meta2024llama3}, and Qwen3-8B~\citep{qwen2025qwen3}, all in FP16. The 7B--13B models run on a single NVIDIA A40 (48\,GB); Vicuna-33B runs on 2$\times$A100-40GB.
We evaluate on five benchmarks spanning code, math, and dialogue:
HumanEval~\citep{chen2021evaluating}, MBPP~\citep{austin2021program}, ClassEval~\citep{du2023classeval}, GSM8K~\citep{cobbe2021training}, and MT-Bench~\citep{zheng2024judging}---ranging from high repetition (code templates) to low repetition (open dialogue).
Dataset sizes and generation lengths are listed in \Cref{tab:datasets}.
Batch size~1; seed~0.

\paragraph{Baselines and metrics.}
AR (standard autoregressive), standalone \pld{}~\citep{saxena2023prompt}, standalone \tr{}~\citep{luo2025tokenrecycling} (using its default imbalanced BFS tree with unigram adjacency), Lookahead Decoding~\citep{fu2024break}, REST~\citep{he2024rest}, and EAGLE-2~\citep{li2024eagle2} (draft-head method; Vicuna-7B, Llama-3-8B, Vicuna-13B, and Vicuna-33B using official released heads).
For the topology comparison (\Cref{sec:topology}), we evaluate \textbf{Isotropic}($k$): a balanced $k$-ary tree with $k \in \{3, 5\}$ and the same node budget, using \method{}'s draft source.
This baseline is a \emph{topology control} (uniform shape at equal budget), not a reproduction of Sequoia, whose DP builds \emph{non-uniform} trees~\citep{chen2024sequoia}; we compare against its official implementation in \Cref{sec:sequoia-headtohead}.
All methods use greedy decoding and produce identical output (lossless).
We additionally compare with EAGLE-3~\citep{li2025eagle3} in \Cref{sec:eagle3}; its drafter is a single decoder layer plus an FC fusion layer, comparable in scale to EAGLE-2's head.
We report $\taumetric$ (compression ratio: mean tokens accepted per verification call; cf.\ \Cref{sec:theory}), tok/s (wall-clock throughput), and Speedup (tok/s vs.\ AR); all hyperparameters are fixed across datasets and models (node budget $B{=}60$, max branch depth $D{=}6$, top-$K{=}10$ adjacency candidates, spine-branch ratio $\rho{=}0.5$; full list in \Cref{app:hyperparams}).
During the first few cycles, the adjacency table is empty; \method{} gracefully degrades to pure \pld{} (or AR if no context match exists) and populates the table from the resulting logits.

\subsection{Main Results}
\label{sec:main-results}

\begin{figure*}[!t]
\centering
\includegraphics[width=\textwidth]{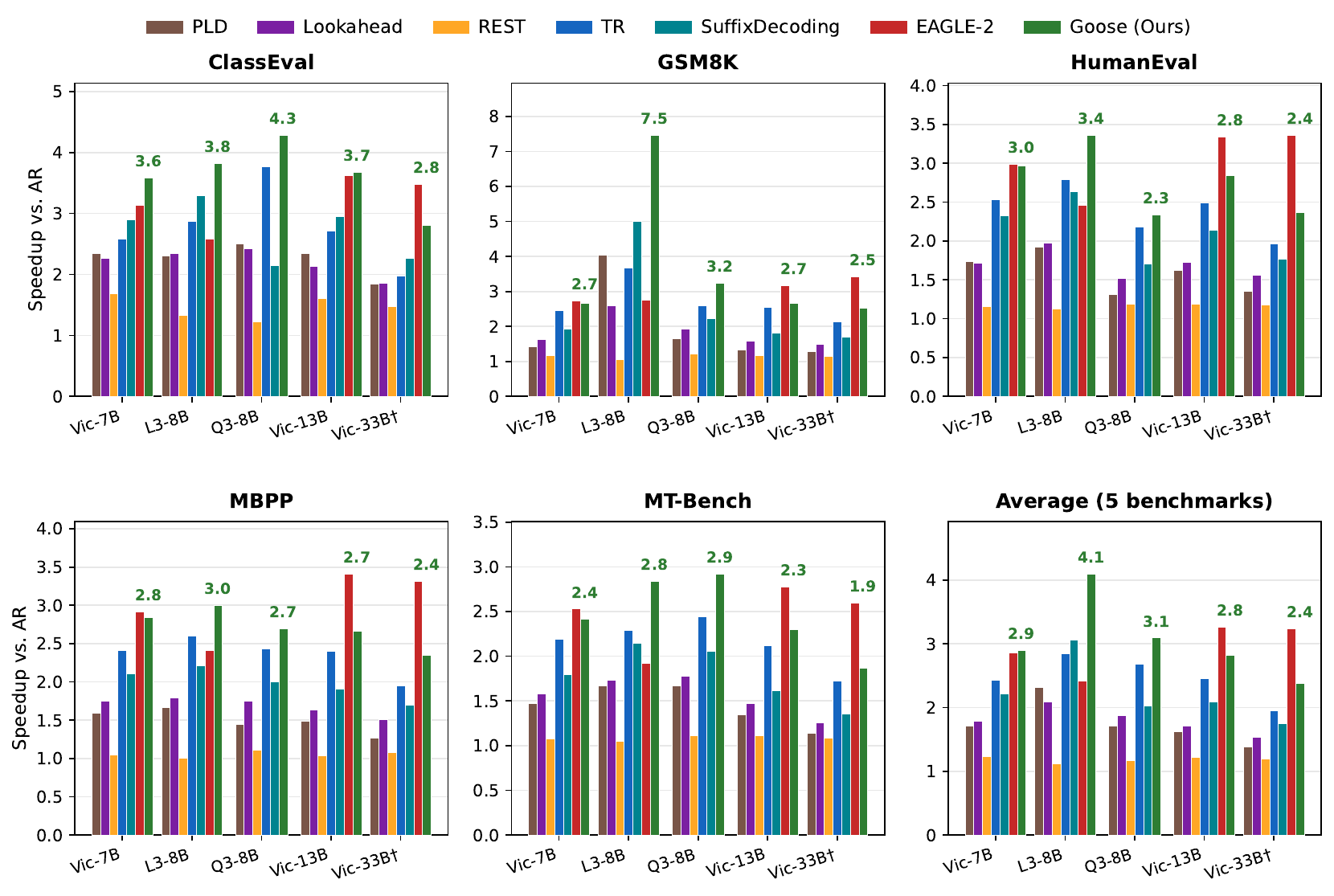}
\caption{Wall-clock speedup over autoregressive (AR) decoding across five models and five benchmarks. All methods are lossless (greedy decoding, identical output); \method{} values are annotated above each group and the sixth panel is the per-model mean. SuffixDecoding is the strongest training-free system reported in prior work, EAGLE-2 the strongest trained one shown (absent on Qwen3-8B, which has no released head). Per-benchmark $\tau$ and Sequoia: \Cref{tab:main-full}; EAGLE-3: \Cref{app:eagle3-raw}. $^\dagger$Vicuna-33B runs on 2$\times$A100-40GB; all others on a single A40.}
\label{fig:main-speedup}
\vspace{-20pt}
\end{figure*}

\Cref{fig:main-speedup} summarizes the main results.
\method{} achieves 1.9--4.3$\times$ lossless speedup across all five models, consistently outperforming every training-free baseline.
Speedup varies predictably with acceptance heterogeneity:
ClassEval (repetitive templates) yields up to 4.3$\times$; MT-Bench (open dialogue) still reaches 1.9--2.9$\times$, confirming that \tr{} branches alone contribute meaningful acceleration.
One outlier: on Llama-3-8B GSM8K, near-verbatim repetition inflates speedup to~7.5$\times$ ($\tau{=}9.25$); all other 24 settings fall within the 1.9--4.3$\times$ speedup range.

\paragraph{Comparison with EAGLE-2.}\looseness=-1
EAGLE-2's trained draft head yields higher $\tau$ (\Cref{tab:main-full}), but its per-step latency offsets the advantage: \method{} outperforms EAGLE-2 on Llama-3-8B (4.1$\times$ vs.\ 2.4$\times$), matches on Vicuna-7B (2.9$\times$ vs.\ 2.9$\times$), while EAGLE-2 leads on Vicuna-13B (2.8$\times$ vs.\ 3.3$\times$) and Vicuna-33B (2.4$\times$ vs.\ 3.2$\times$).
Two factors explain this pattern.
\emph{Model size sets the wall-clock winner}: at ${\leq}8$B the base forward is cheap, so zero drafter cost outweighs EAGLE's higher $\tau$; at ${\geq}13$B the slower base forward lets that higher $\tau$ convert and win.
\emph{Task repetition sets the margin}: the spine fires most on repetitive workloads and least on open-ended dialogue.
We read this as an empirical trend on these models, not a strict size law.
As a training-free method, \method{} also covers models for which no trained draft head exists (Qwen3-8B).

\paragraph{Comparison with stronger baselines.}\looseness=-1
\label{sec:review-additions}\label{sec:sequoia-headtohead}\label{sec:suffix}\label{sec:eagle3}
Three further systems bracket the strongest alternatives; SuffixDecoding and Sequoia also appear in \Cref{tab:main-full}.
SuffixDecoding~\citep{oliaro2024suffix}, the strongest training-free drafter reported in prior work, drafts from a suffix tree over the prompt and prior outputs---richer statistics than \method{}'s bigram adjacency table.
At full scale (6 models $\times$ 13 datasets), \method{} leads $\tau$ in all 78 cells and wall-clock throughput in all 52 cells with full accounting (geometric mean $+35$--$36\%$); per-cycle instrumentation attributes the gap mainly to depth allocation---\method{} speculates 52--60 nodes per non-dry cycle at depths $5.7$--$22$, against SuffixDecoding's 3--8 at $2.9$--$7.3$---with coverage (twice as many dry cycles, whose draft is empty or a single token) and cold start second, and merging its two trees into one does not close it (\Cref{app:suffix-mech}).
Conversely, swapping our \pld{} spine for that suffix tree raises $\tau$ in all 15 cells tested ($+3$--$47\%$, \Cref{app:cross-source}), so the topology is agnostic to which source feeds the spine.
Sequoia~\citep{chen2024sequoia} optimizes tree shape under a single acceptance vector: we reproduce its published $\tau$ within $\pm 6\%$ on its released Llama-2-7B stack, where its trained 68M drafter wins $\tau$ on all three of its datasets yet the two systems trade wall-clock (\method{} $+8\%$ on C4, Sequoia $+5\%$ on CNN/DM, OpenWebText within noise). Its per-cycle drafter forward decides the outcome and costs more at modern scale: there Sequoia reaches $\tau$ of $6.2$--$9.2$ yet never exceeds $1.60\times$ wall-clock, and on Qwen3-8B it drops below AR outright (\Cref{app:sequoia-modern}).
EAGLE-3~\citep{li2025eagle3} is the strongest trained drafter.
Its head is trained on chat-template data, so the chat-template protocol is its fair setting; there it wins 4 of 5 datasets ($+53$--$79\%$ tok/s, with \method{} retaining ClassEval at $+39\%$).
Under our raw-prompt protocol the result inverts and \method{} leads all five ($+25$--$83\%$ on the four non-truncated datasets; \Cref{app:eagle3-raw}).
We do not claim to beat EAGLE-3 in its native setting; the two are cross-category (\Cref{sec:training-free-value}).

\subsection{Anisotropic vs.\ Isotropic Topology and Ablation}
\label{sec:topology}

\begin{table*}[!t]
\centering
% --- Caption row (top-aligned; shorter caption padded by taller one) ---
\begin{minipage}[t]{0.47\textwidth}
\vspace{0pt}\centering\small
\caption{Compression ratio~$\tau$ for spine vs.\ isotropic trees (budget $B{=}60$). $\Delta_3$: relative gain over Iso(3).}
\label{tab:topology}
\end{minipage}%
\hfill
\begin{minipage}[t]{0.51\textwidth}
\vspace{0pt}\centering\small
\caption{Ablation on Qwen3-8B ($\tau$). Each row removes one design choice; $\bar\Delta$: macro-averaged relative change.}
\label{tab:ablation}
\end{minipage}
\vspace{-2pt}
% --- Body row (both table bodies start at the same height) ---
\noindent
\begin{minipage}[t]{0.47\textwidth}
\centering\small
\resizebox{\linewidth}{!}{
\begin{tabular}{@{}ll rrr r@{}}
\toprule
\textbf{Model} & \textbf{Dataset} & \textbf{Spine (Ours)} & \textbf{Iso(3)} & \textbf{Iso(5)} & $\boldsymbol{\Delta_3}$ \\
\midrule
\multirow{3}{*}{L3-8B}
& ClassEval  & \textbf{4.89} & 3.67 & 3.60 & +33\% \\
& HumanEval  & \textbf{4.11} & 3.32 & 3.27 & +24\% \\
& MT-Bench   & \textbf{3.50} & 2.76 & 2.76 & +27\% \\
\midrule
\multirow{5}{*}{Q3-8B}
& ClassEval  & \textbf{5.19} & 3.97 & 3.92 & +31\% \\
& GSM8K      & \textbf{3.87} & 3.00 & 3.01 & +29\% \\
& HumanEval  & \textbf{2.78} & 2.49 & 2.51 & +12\% \\
& MBPP       & \textbf{3.20} & 2.73 & 2.77 & +17\% \\
& MT-Bench   & \textbf{3.51} & 2.81 & 2.81 & +25\% \\
\midrule
\multicolumn{2}{@{}l}{Average} & \textbf{3.88} & 3.09 & 3.08 & +25.4\% \\
\bottomrule
\end{tabular}
}
\end{minipage}%
\hfill
\begin{minipage}[t]{0.51\textwidth}
\centering\small
\resizebox{\linewidth}{!}{
\begin{tabular}{@{}l rrrrr r@{}}
\toprule
\textbf{Config.} & \textbf{HE} & \textbf{GS} & \textbf{MB} & \textbf{CE} & \textbf{MT} & $\boldsymbol{\bar\Delta}$ \\
\midrule
\method{} (Ours)              & \textbf{2.78} & \textbf{3.87} & \textbf{3.20} & \textbf{5.19} & \textbf{3.51} & --- \\
\midrule
\multicolumn{7}{@{}l}{\textit{Idea 1: Anisotropic topology (\S\ref{sec:tree})}} \\
\quad w/o spine branches      & 2.71 & 3.70 & 3.10 & 5.03 & 3.46 & $-$2.9\% \\
\quad w/o bigram adjacency    & 2.69 & 3.66 & 3.11 & 5.01 & 3.38 & $-$3.7\% \\
\midrule
\multicolumn{7}{@{}l}{\textit{Idea 2: Confidence-adaptive budget (\S\ref{sec:consensus})}} \\
\quad w/o consensus bypass    & 2.77 & 3.83 & 3.05 & 4.60 & 3.22 & $-$5.1\% \\
\quad w/o context-match spine & 2.68 & 3.76 & 3.11 & 5.00 & 3.37 & $-$3.4\% \\
\bottomrule
\end{tabular}
}
\end{minipage}
\vspace{-12pt}
\end{table*}

\Cref{tab:topology} confirms the prediction of \Cref{sec:theory}: the spine tree wins all 8 settings by 12--33\% (we report Llama-3-8B and Qwen3-8B as architecturally distinct representatives; Vicuna variants show similar trends).
Iso($k{=}5$) performs similarly to Iso($k{=}3$), showing that the isotropic regime itself---not $k$---is the bottleneck.

\paragraph{Ablation analysis.}
\Cref{tab:ablation} isolates each design idea by removing one component at a time on Qwen3-8B.
\textbf{Idea~1 (Anisotropic topology).}~Removing spine branches ($-2.9\%$) or bigram adjacency ($-3.7\%$) each degrades $\tau$, confirming that the advantage requires both branch placement and transition quality.
\textbf{Idea~2 (Confidence-adaptive budget).}~Consensus bypass is the largest contributor ($-5.1\%$, up to $-11.4\%$ on ClassEval): when all $n$-gram lengths agree on a long continuation, \method{} extends the spine directly without tree construction.
Removing the spine ($-3.4\%$) shows the \pld{} signal itself is essential; topology and adaptation are complementary.
\textbf{Idea~3 (Cross-source spine continuation).}~Disabling spine continuation---the greedy walk's extension from a broken spine into a \tr{} branch (\Cref{sec:continuation})---while keeping the anisotropic shape costs $-4.4\%$ on average ($-1.7\%$ HumanEval to $-6.7\%$ MT-Bench), measured as a matched pair (both arms re-run on identical code and configuration).
The tree-\emph{shape} gain itself is measured separately and assumption-free in \Cref{tab:topology} ($+12$--$33\%$ at equal budget over the isotropic control).
To complement the ablation, we note that moving from \tr{}-only isotropic to full \method{} on Qwen3-8B (MT-Bench) raises $\tau$ from 2.81 to 3.51 ($+$25\%); per-component contributions are the leave-one-out deltas of \Cref{tab:ablation}.

\subsection{Analysis}
\label{sec:analysis-main}

\paragraph{Source decomposition and synergy.}
\Cref{fig:heterogeneity-decomp}a decomposes each model's mean~$\tau$ into the best standalone baseline $\max(\tau_{\pld{}}, \tau_{\tr{}})$ and the synergy gain from the spine-tree topology.
The synergy ratio ($\tau_{\method{}} / \max(\tau_{\pld}, \tau_{\tr})$) exceeds 1.0 on all 25 settings (average 1.24$\times$; per-model breakdown in \Cref{app:synergy}).
The source split varies by task; even on the weakest datasets, every A40 cell stays above $2.6\times$ $\tau$ (\Cref{app:robustness}).
A per-accept source attribution across four models confirms the sources are complementary: \pld{}'s share spans 33--85\% by task, and where it is sparsest \tr{} supplies the majority without $\tau$ collapsing (\Cref{app:source-attribution}).
Tree construction adds negligible CPU overhead ($<$1\% of wall-clock time).
\method{} has no neural drafter: the draft comes from table lookup, and the only model call per cycle is the verification forward.
This makes the conversion from $\tau$ to wall-clock tight.
On Llama-3-8B HumanEval, $\tau{=}4.11$ yields a measured $3.36\times$ speedup, a ${\approx}82\%$ conversion; the shortfall is tree construction plus the slightly larger, weight-bound verification forward.

% ═══════════════════════════════════════════════════════════════════════
% DISCUSSION AND CONCLUSION
% ═══════════════════════════════════════════════════════════════════════
\section{Discussion and Conclusion}
\label{sec:discussion}

We introduced the spine tree, an anisotropic draft-tree topology that exploits acceptance heterogeneity among training-free draft sources.
Without any training, \method{} achieves 1.9--4.3$\times$ lossless speedup, outperforms isotropic baselines by 12--33\% (\Cref{tab:topology}), and matches or exceeds EAGLE-2 wall-clock speed on two of the four models with published heads.

\paragraph{Where training-free speculation is valuable.}
\label{sec:training-free-value}\looseness=-1
Trained-drafter and training-free methods are cross-category.
A trained head buys raw $\tau$ where its training distribution matches; a training-free drafter buys zero-cost, day-one deployment and still wins outright at ${\leq}8$B (Llama-3-8B HumanEval: $3.36\times$ vs.\ EAGLE-2's $2.46\times$).
The trade-off flips at ${\geq}13$B, where the slower base forward lets EAGLE-2's higher $\tau$ convert ($3.24\times$ vs.\ \method{}'s $2.38\times$ on Vicuna-33B; \Cref{tab:main-full}).

\paragraph{Limitations.}\looseness=-1
Our main results use batch size~1, the latency-sensitive, memory-bandwidth-bound regime where speculation is strongest.
A batch-size sweep (\Cref{app:batch}) shows \method{} leading at small batch and batched AR overtaking as batch size grows (crossover ${\approx}$4--16), the general speculation-vs-batching trade-off.
The adjacency table stays small ($<$7\,MB), and our analysis assumes greedy decoding; production-grade batched serving and sampling-based verification~\citep{leviathan2023fast} remain future work.

\newpage
\section*{Acknowledgments}
This work was supported by the Nakajima Foundation.
The authors sincerely thank the chairs and reviewers of COLM 2026 for their thoughtful and insightful feedback; their reviews and the discussion period greatly improved this work.

% ═══════════════════════════════════════════════════════════════════════
\bibliography{colm2026_conference}
\bibliographystyle{colm2026_conference}

% Force all floats to appear before the appendix
\FloatBarrier

% ═══════════════════════════════════════════════════════════════════════
% APPENDIX
% ═══════════════════════════════════════════════════════════════════════
\appendix

\section{Full Main Results}
\label{app:full-main}

\begin{table*}[!ht]
\centering
\caption{Full main results. Speedup (Spd): wall-clock tok/s ratio vs.\ AR measured in the same run. $\tau$: compression ratio (tokens per target-model forward; drafter forwards are not counted, which favors the trained rows). \textbf{Bold}: best per column. All methods produce identical output (lossless). \colorbox{gray!15}{Shaded rows} denote methods requiring per-model training. Sequoia appears on the two targets with released modern drafters, at $N{=}60$ to match our budget $B$; its $\tau$ follows Sequoia's official convention (prefill call excluded), at most $1.7\%$ above the table's convention. EAGLE-2 is absent on Qwen3-8B, which has no released head. $^\dagger$Vicuna-33B evaluated on NVIDIA A100 (2$\times$40GB); all other models on A40. The Vicuna-33B GSM8K EAGLE-2 cell comes from a separate run of that environment, so its Spd is taken against that run's own AR; $\tau$ is unaffected.}
\label{tab:main-full}
\resizebox{\textwidth}{!}{%
\small
\setlength{\tabcolsep}{3pt}
\begin{tabular}{l l rr rr rr rr rr r}
\toprule
& & \multicolumn{2}{c}{\textbf{ClassEval}} & \multicolumn{2}{c}{\textbf{GSM8K}} & \multicolumn{2}{c}{\textbf{HumanEval}} & \multicolumn{2}{c}{\textbf{MBPP}} & \multicolumn{2}{c}{\textbf{MT-Bench}} & \textbf{Avg.} \\
\cmidrule(lr){3-4} \cmidrule(lr){5-6} \cmidrule(lr){7-8} \cmidrule(lr){9-10} \cmidrule(lr){11-12} \cmidrule(lr){13-13}
\textbf{Model} & \textbf{Method} & Spd & $\tau$ & Spd & $\tau$ & Spd & $\tau$ & Spd & $\tau$ & Spd & $\tau$ & Spd \\
\midrule
\multirow{8}{*}{\rotatebox{90}{\small Vicuna-7B}}
& AR                         & 1.00 & 1.00 & 1.00 & 1.00 & 1.00 & 1.00 & 1.00 & 1.00 & 1.00 & 1.00 & 1.00 \\
& \pld{}                     & 2.34 & 2.50 & 1.43 & 1.46 & 1.74 & 1.78 & 1.59 & 1.62 & 1.47 & 1.48 & 1.71 \\
& \tr{}                      & 2.58 & 3.29 & 2.46 & 2.96 & 2.53 & 3.06 & 2.41 & 2.86 & 2.19 & 2.60 & 2.43 \\
& Lookahead                  & 2.27 & 2.41 & 1.63 & 1.67 & 1.72 & 1.77 & 1.75 & 1.78 & 1.58 & 1.60 & 1.79 \\
& REST                       & 1.68 & 1.73 & 1.17 & 1.19 & 1.16 & 1.17 & 1.05 & 1.05 & 1.08 & 1.08 & 1.23 \\
& SuffixDecoding             & 2.90 & 3.34 & 1.92 & 1.98 & 2.33 & 2.44 & 2.11 & 2.16 & 1.80 & 1.85 & 2.21 \\
\rowcolor{gray!15}
& EAGLE-2                    & 3.14 & \textbf{5.50} & \textbf{2.73} & \textbf{4.51} & \textbf{2.99} & \textbf{5.02} & \textbf{2.92} & \textbf{4.83} & \textbf{2.53} & \textbf{4.21} & 2.86 \\
& \method{} (Ours)           & \textbf{3.58} & 4.64 & 2.67 & 3.23 & 2.97 & 3.63 & 2.84 & 3.40 & 2.42 & 2.88 & \textbf{2.90} \\
\midrule
\multirow{9}{*}{\rotatebox{90}{\small Llama-3-8B}}
& AR                         & 1.00 & 1.00 & 1.00 & 1.00 & 1.00 & 1.00 & 1.00 & 1.00 & 1.00 & 1.00 & 1.00 \\
& \pld{}                     & 2.30 & 2.49 & 4.04 & 4.18 & 1.92 & 2.01 & 1.67 & 1.72 & 1.67 & 1.72 & 2.32 \\
& \tr{}                      & 2.87 & 3.64 & 3.68 & 4.52 & 2.79 & 3.40 & 2.60 & 3.11 & 2.29 & 2.82 & 2.85 \\
& Lookahead                  & 2.35 & 2.57 & 2.59 & 2.80 & 1.98 & 2.08 & 1.79 & 1.86 & 1.73 & 1.80 & 2.09 \\
& REST                       & 1.33 & 1.35 & 1.06 & 1.07 & 1.13 & 1.14 & 1.01 & 1.01 & 1.05 & 1.05 & 1.12 \\
& SuffixDecoding             & 3.29 & 3.46 & 5.00 & 5.10 & 2.64 & 2.72 & 2.21 & 2.26 & 2.15 & 2.18 & 3.06 \\
\rowcolor{gray!15}
& Sequoia                    & 1.52 & \textbf{8.63} & 1.56 & 8.89 & 1.53 & \textbf{8.56} & 1.48 & \textbf{8.42} & 1.36 & \textbf{7.57} & 1.49 \\
\rowcolor{gray!15}
& EAGLE-2                    & 2.58 & 5.29 & 2.75 & 5.52 & 2.46 & 4.95 & 2.41 & 4.83 & 1.92 & 3.85 & 2.42 \\
& \method{} (Ours)           & \textbf{3.82} & 4.89 & \textbf{7.46} & \textbf{9.25} & \textbf{3.36} & 4.11 & \textbf{3.00} & 3.61 & \textbf{2.84} & 3.50 & \textbf{4.10} \\
\midrule
\multirow{8}{*}{\rotatebox{90}{\small Qwen3-8B}}
& AR                         & 1.00 & 1.00 & 1.00 & 1.00 & 1.00 & 1.00 & 1.00 & 1.00 & 1.00 & 1.00 & 1.00 \\
& \pld{}                     & 2.50 & 2.61 & 1.65 & 1.71 & 1.31 & 1.36 & 1.45 & 1.50 & 1.67 & 1.73 & 1.72 \\
& \tr{}                      & 3.77 & 4.50 & 2.60 & 3.03 & 2.18 & 2.55 & 2.43 & 2.81 & 2.44 & 2.86 & 2.68 \\
& Lookahead                  & 2.42 & 2.42 & 1.93 & 1.88 & 1.52 & 1.47 & 1.75 & 1.68 & 1.78 & 1.74 & 1.88 \\
& REST                       & 1.22 & 1.14 & 1.21 & 1.16 & 1.19 & 1.12 & 1.11 & 1.04 & 1.11 & 1.03 & 1.17 \\
& SuffixDecoding             & 2.15 & 2.21 & 2.22 & 2.38 & 1.71 & 1.68 & 2.00 & 1.94 & 2.06 & 2.18 & 2.03 \\
\rowcolor{gray!15}
& Sequoia                    & 0.82 & \textbf{7.32} & 0.85 & \textbf{7.58} & 0.76 & \textbf{6.81} & 0.82 & \textbf{7.34} & 0.70 & \textbf{6.22} & 0.79 \\
& \method{} (Ours)           & \textbf{4.28} & 5.19 & \textbf{3.23} & 3.87 & \textbf{2.34} & 2.78 & \textbf{2.70} & 3.20 & \textbf{2.92} & 3.51 & \textbf{3.09} \\
\midrule
\multirow{8}{*}{\rotatebox{90}{\small Vicuna-13B}}
& AR                         & 1.00 & 1.00 & 1.00 & 1.00 & 1.00 & 1.00 & 1.00 & 1.00 & 1.00 & 1.00 & 1.00 \\
& \pld{}                     & 2.34 & 2.47 & 1.33 & 1.36 & 1.62 & 1.66 & 1.49 & 1.51 & 1.35 & 1.36 & 1.63 \\
& \tr{}                      & 2.71 & 3.34 & 2.56 & 3.02 & 2.49 & 2.94 & 2.40 & 2.78 & 2.12 & 2.46 & 2.46 \\
& Lookahead                  & 2.14 & 2.26 & 1.58 & 1.62 & 1.73 & 1.78 & 1.64 & 1.67 & 1.47 & 1.49 & 1.71 \\
& REST                       & 1.61 & 1.66 & 1.18 & 1.21 & 1.19 & 1.20 & 1.04 & 1.04 & 1.11 & 1.11 & 1.23 \\
& SuffixDecoding             & 2.95 & 3.36 & 1.82 & 1.90 & 2.14 & 2.26 & 1.91 & 1.97 & 1.62 & 1.68 & 2.09 \\
\rowcolor{gray!15}
& EAGLE-2                    & 3.62 & \textbf{6.03} & \textbf{3.18} & \textbf{4.98} & \textbf{3.34} & \textbf{5.24} & \textbf{3.41} & \textbf{5.30} & \textbf{2.78} & \textbf{4.33} & \textbf{3.27} \\
& \method{} (Ours)           & \textbf{3.68} & 4.59 & 2.66 & 3.14 & 2.84 & 3.37 & 2.66 & 3.11 & 2.30 & 2.67 & 2.83 \\
\midrule
\multirow{8}{*}{\rotatebox{90}{\small Vicuna-33B$^\dagger$}}
& AR                         & 1.00 & 1.00 & 1.00 & 1.00 & 1.00 & 1.00 & 1.00 & 1.00 & 1.00 & 1.00 & 1.00 \\
& \pld{}                     & 1.85 & 1.99 & 1.29 & 1.33 & 1.36 & 1.40 & 1.27 & 1.30 & 1.14 & 1.15 & 1.38 \\
& \tr{}                      & 1.97 & 2.62 & 2.14 & 2.63 & 1.97 & 2.37 & 1.95 & 2.29 & 1.72 & 2.01 & 1.95 \\
& Lookahead                  & 1.86 & 2.02 & 1.49 & 1.55 & 1.56 & 1.56 & 1.51 & 1.49 & 1.26 & 1.23 & 1.54 \\
& REST                       & 1.48 & 1.54 & 1.14 & 1.18 & 1.18 & 1.16 & 1.08 & 1.05 & 1.09 & 1.05 & 1.19 \\
& SuffixDecoding             & 2.26 & 2.59 & 1.69 & 1.81 & 1.77 & 1.82 & 1.70 & 1.72 & 1.36 & 1.37 & 1.76 \\
\rowcolor{gray!15}
& EAGLE-2                    & \textbf{3.48} & \textbf{5.15} & \textbf{3.42} & \textbf{4.73} & \textbf{3.36} & \textbf{4.65} & \textbf{3.32} & \textbf{4.55} & \textbf{2.60} & \textbf{3.58} & \textbf{3.24} \\
& \method{} (Ours)           & 2.80 & 3.69 & 2.53 & 3.10 & 2.37 & 2.83 & 2.35 & 2.74 & 1.87 & 2.15 & 2.38 \\
\bottomrule
\end{tabular}%
}
\end{table*}

\noindent \Cref{tab:main-full} reports compression ratio~($\tau$) and wall-clock speedup (Spd) for all five models across five benchmarks.
\method{} achieves the highest average speedup on three of five models; EAGLE-2 leads on Vicuna-13B and Vicuna-33B thanks to its trained draft head (official EAGLE heads throughout) and has no published head for Qwen3-8B.
For the Vicuna-33B EAGLE-2 GSM8K cell, speedup is computed against the AR baseline of its own run environment; $\tau$ is environment-independent.
The per-benchmark $\tau$ values are the ground truth for the synergy analysis in \Cref{app:synergy}.

\FloatBarrier

\section{Algorithm Details}
\label{app:algorithms}

\subsection{Spine Tree Construction}
\label{app:build-spine}

\begin{algorithm}[h]
\small
\caption{\textsc{BuildSpineTree}}
\label{alg:build-spine}
\begin{algorithmic}[1]
\Require Anchor token $t_a$;\; \pld{} draft chain $\mathbf{d}=(d_1,\dots,d_m)$
\Require Adjacency table $\mathcal{A}$;\; node budget $B$;\; max depth $D$
\Require Spine ratio $r \in (0,1)$;\; spine-branch ratio $\rho \in (0,1)$
\Ensure Spine tree $\mathcal{T}$ with $\leq B$ nodes and tree attention mask
\Statex
\State $b_s \leftarrow \min(m,\,\lfloor B r \rfloor)$ \Comment{PLD spine node count}
\State $b_r \leftarrow \lfloor (B{-}1{-}b_s)(1{-}\rho) \rfloor$ \Comment{Root-level branch count}
\State $b_\rho \leftarrow B - 1 - b_s - b_r$ \Comment{Spine-level branch count}
\Statex
\State $d_0 \leftarrow t_a$ \Comment{Anchor as spine root}
\For{$i = 1$ \textbf{to} $b_s$} \Comment{Step 1: lay PLD spine}
\State Add $d_i$ as child of $d_{i-1}$
\EndFor
\Statex
\State Attach top-$b_r$ successors from $\mathcal{A}(t_a)$ to root \Comment{Step 2: root branches}
\Statex
\For{$i = 1$ \textbf{to} $b_s$} \Comment{Step 3: spine branches (harmonic decay)}
\State $a_i \leftarrow \bigl\lfloor b_\rho \cdot \tfrac{1/i}{\sum_{j=1}^{b_s} 1/j} \bigr\rfloor$ \Comment{Allocation for spine node $d_i$}
\State Attach top-$a_i$ successors from $\mathcal{A}(d_i)$ to $d_i$
\EndFor
\Statex
\State BFS-extend all branch leaves through $\mathcal{A}$ up to depth $D$ \Comment{Step 4}
\State \Return $\mathcal{T}$
\end{algorithmic}
\end{algorithm}

\noindent Duplicate entries (e.g., spine continuation $d_{i+1}$ already present in $\mathcal{A}(d_i)$) are excluded during attachment to avoid redundant nodes.
The BFS extension in Step~4 is a standard breadth-first expansion: starting from each branch leaf, it greedily appends top-scoring successors from $\mathcal{A}$ up to depth~$D$ within the remaining budget.

\subsection{Decoding Loop}
\label{app:algorithm}

\begin{algorithm}[h]
\caption{\method{} Decoding Loop}
\label{alg:goose}
\begin{algorithmic}[1]
\Require Target model $\mathcal{M}$;\; prompt $\mathbf{x}$;\; max tokens $N$
\Require Budget $B$;\; spine-branch ratio $\rho$;\; max depth $D$;\; bypass threshold $\ell_{\text{byp}}$;\; adjacency width $K$
\Ensure Generated tokens $\mathbf{y}$, identical to autoregressive output
\Statex
\State Build context index from $\mathbf{x}$;\; init adjacency table $\mathcal{A}$ with top-$K$ width
\State $\boldsymbol{\ell} \leftarrow \mathcal{M}(\mathbf{x})$;\; update $\mathcal{A}$ with top-$K$ entries per position from $\boldsymbol{\ell}$
\State $t_{\text{anc}} \leftarrow \argmax \boldsymbol{\ell}[-1]$;\; $\hat{p}_s \leftarrow 0.3$ \Comment{Running estimate of spine acceptance $p_s$}
\While{$|\mathbf{y}| < N$ and not \texttt{EOS}}
\State $(\mathbf{d}, c) \leftarrow \textsc{ContextMatch}(\mathbf{y})$ \Comment{Draft chain $\mathbf{d}$; consensus flag $c$}
\If{$\mathbf{d} \neq \emptyset$ \textbf{and} ($|\mathbf{d}| \geq \ell_{\text{byp}}$ \textbf{or} $c$)}
\State $(\mathbf{y}_{\text{acc}},\, b) \leftarrow$ linear-verify $\mathbf{d}$ with $\mathcal{M}$ \Comment{Bypass: high-confidence PLD}
\ElsIf{$\mathbf{d} \neq \emptyset$ \textbf{or} $\mathcal{A}$ has successors for $t_{\text{anc}}$}
\State $r \leftarrow \textsc{SpineRatioTier}(\hat{p}_s)$ \Comment{Adaptive spine ratio (\Cref{sec:consensus})}
\State $\mathcal{T} \leftarrow$ \Call{BuildSpineTree}{$t_{\text{anc}}, \mathbf{d}, \mathcal{A}, r, B, \rho, D$}
\State $(\mathbf{y}_{\text{acc}},\, b) \leftarrow$ \Call{UnifiedGreedyWalk}{$\mathcal{M}, \mathcal{T}$}
\Else
\State $b \leftarrow \argmax\, \mathcal{M}(t_{\text{anc}})$;\; $\mathbf{y}_{\text{acc}} \leftarrow \emptyset$ \Comment{AR fallback}
\EndIf
\State Update $\mathcal{A}$ from all logits of this cycle \Comment{Including rejected branches}
\State Update $\hat{p}_s$ via EMA;\; $t_{\text{anc}} \leftarrow b$;\; append $\mathbf{y}_{\text{acc}} \| b$ to $\mathbf{y}$
\EndWhile
\end{algorithmic}
\end{algorithm}

\method{} routes each cycle by source availability and confidence (\Cref{sec:consensus}): (1)~high-confidence or long context match triggers linear bypass; (2)~any draft source available triggers the spine tree with unified greedy walk verification; (3)~no source available triggers AR fallback.
When only transitions (no context match) are available, $\mathbf{d} = \emptyset$ and the tree degenerates to a \tr{}-only tree; when only a context match is available, the tree reduces to a linear chain. Both are special cases that preserve the expected-yield guarantee of \Cref{prop:monotonicity}.

\textsc{ContextMatch} (called at line~5 of \Cref{alg:goose}) is a standard $n$-gram lookup: given the generated sequence and a set of query lengths $\{3,4,5\}$, it finds the longest matching prefix in the context, copies the continuation as the draft chain~$\mathbf{d}$, and returns a consensus flag~$c$ (true iff $\geq 2$ matches agree on the first continuation token).

\subsection{Unified Greedy Walk Verification}
\label{app:verify}

\begin{algorithm}[h]
\small
\caption{\textsc{UnifiedGreedyWalk}: Verification with Source Priority}
\label{alg:verify}
\begin{algorithmic}[1]
\Require Model $\mathcal{M}$;\; spine tree $\mathcal{T}$ (nodes tagged \pld{} or \tr{})
\Ensure Accepted token list;\; bonus token $b$
\State $\boldsymbol{\ell} \leftarrow \mathcal{M}(\mathcal{T})$ with tree attention mask \Comment{Single forward pass}
\State $\hat{x}_i \leftarrow \argmax \boldsymbol{\ell}_i$ for each node $i$ \Comment{Greedy prediction per node}
\Statex
\State $v \leftarrow \text{root}$;\; accepted $\leftarrow []$
\While{$v$ has children in $\mathcal{T}$}
\If{$\exists$ \pld{}-child $v'$ of $v$: $\mathcal{T}[v'] = \hat{x}_v$} \Comment{Priority 1: spine}
\State accepted.append($v'$);\; $v \leftarrow v'$;\; \textbf{continue}
\EndIf
\If{$\exists$ \tr{}-child $v'$ of $v$: $\mathcal{T}[v'] = \hat{x}_v$} \Comment{Priority 2: branch}
\State accepted.append($v'$);\; $v \leftarrow v'$;\; \textbf{continue}
\EndIf
\State \textbf{break} \Comment{No child matches}
\EndWhile
\Statex
\State $b \leftarrow \hat{x}_{v}$ \Comment{Bonus token from last accepted position}
\State \Return accepted, $b$
\end{algorithmic}
\end{algorithm}

At each tree node, the walk checks \pld{} (spine) children before \tr{} (branch) children, giving priority to the high-acceptance source.
Once the walk transitions from spine to branch, it cannot return, producing the ``spine continuation'' path described in \Cref{sec:continuation}.
The walk naturally yields the longest accepted path and covers all three outcome categories: pure PLD (spine only), spine continuation (PLD prefix then TR suffix), or pure TR (root branch only).

\clearpage
\section{Implementation and Setup}
\label{app:impl-setup}

\subsection{Implementation Details}
\label{app:impl}

This section provides the full engineering details summarized in \Cref{sec:consensus}.

\paragraph{Dense GPU bigram adjacency table.}
The adjacency table $\mathcal{A}$ (\Cref{sec:tree}) is implemented as a two-tier structure (dense GPU unigram tensors plus a bigram hash table).
\textbf{Tier 1 (Unigram):} Two dense tensors of shape $|V|{\times}K$ on GPU store, respectively, the token IDs and logit-derived scores of the top-$K$ successors for each vocabulary token; lookup is $O(1)$ with no CPU transfer.
\textbf{Tier 2 (Bigram):} A hash table mapping $(t_{i-1}, t_i) \mapsto \text{top-}K$ successors, exploiting the fact that two-token context substantially sharpens predictions (e.g., successors of \texttt{f} after \texttt{def} are much more predictable than \texttt{f} alone).

\paragraph{Confidence-aware branch width.}
Branch width scales with the transition score in $\mathcal{A}$: each successor's allocated children (from the $1/i$ harmonic schedule in Step~3 of \Cref{alg:build-spine}) are further modulated by its score relative to its siblings, and successors whose score falls below 0.01 are pruned entirely.
This focuses the node budget on high-confidence successors while keeping the total tree size within $B$.

\paragraph{Aggressive logit harvesting.}
Every forward pass's logits, including those at \textit{rejected branches} and at prefill, are harvested into the adjacency table via a batched \texttt{topk} ($<$200\,\textmu{}s overhead).
This populates $\mathcal{A}$ far faster than the verified-positions-only strategy used in prior work.

\subsection{Datasets and Model Details}
\label{app:datasets}

\begin{table}[h]
\centering
\small
\caption{Benchmark details.}
\label{tab:datasets}
\begin{tabular}{@{}l l r r@{}}
\toprule
\textbf{Benchmark} & \textbf{Domain} & \textbf{Samples} & \textbf{Max tokens} \\
\midrule
HumanEval & Code & 164 & 512 \\
MBPP & Code & 500 & 512 \\
ClassEval & Code & 100 & 512 \\
GSM8K & Math & 1{,}319 & 1{,}024 \\
MT-Bench & Dialogue & 80 & 1{,}024 \\
\bottomrule
\end{tabular}
\end{table}

\subsection{Hyperparameters and Experimental Setup}
\label{app:hyperparams}

\Cref{tab:hyperparams} lists all hyperparameters and \Cref{tab:hw} the hardware and software configuration.

\begin{table*}[!t]
\centering
\begin{minipage}[t]{0.48\textwidth}
\vspace{0pt}
\centering
\small
\caption{All hyperparameters (fixed across datasets and models).}
\label{tab:hyperparams}
\resizebox{\linewidth}{!}{
\begin{tabular}{lr}
\toprule
\textbf{Parameter} & \textbf{Value} \\
\midrule
Context match n-gram lengths      & ${3, 4, 5}$ \\
Max spine continuation     & 20 \\
Transition top-$K$        & 10 \\
Tree node budget $B$     & 60 \\
Max branch (transition) depth $D$ & 6 \\
Min score threshold      & 0.01 \\
Spine-branch ratio $\rho$    & 0.5 \\
EMA smoothing coefficient       & 0.3 \\
Spine ratio tiers $(r)$      & $\hat{p}_s{<}0.2{\to}0.15$, ${<}0.4{\to}0.30$, ${\geq}0.4{\to}0.50$ \\
Linear bypass threshold  & 8 tokens \\
\bottomrule
\end{tabular}
}
\end{minipage}%
\hfill
\begin{minipage}[t]{0.48\textwidth}
\vspace{0pt}
\centering
\small
\caption{Hardware and software configuration.}
\label{tab:hw}
\resizebox{\linewidth}{!}{
\begin{tabular}{l l}
\toprule
\textbf{Item} & \textbf{Specification} \\
\midrule
GPU (7B--13B)   & NVIDIA A40 (48\,GB GDDR6) \\
GPU (33B)       & 2$\times$A100-PCIE-40GB (80\,GB total) \\
CUDA Toolkit    & 12.8 \\
PyTorch         & 2.9.1+cu128 \\
Transformers    & 4.57 (Hugging Face) \\
Precision       & FP16 (\texttt{float16}) \\
\bottomrule
\end{tabular}
}
\end{minipage}
\end{table*}

\section{Extended Experimental Analysis}
\label{app:ext-analysis}

\subsection{Synergy Decomposition and Hyperparameter Sensitivity}
\label{app:synergy}

\begin{figure}[h]
\centering
\begin{minipage}[t]{0.50\textwidth}
\centering
\includegraphics[width=\linewidth]{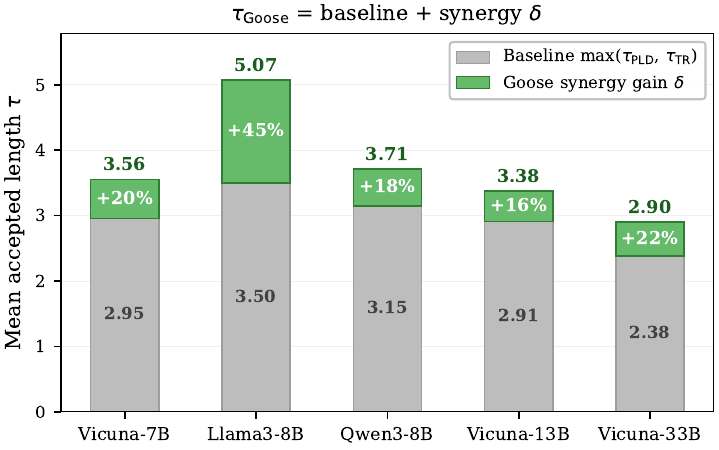}
\end{minipage}%
\hfill
\begin{minipage}[t]{0.48\textwidth}
\centering
\includegraphics[width=\linewidth]{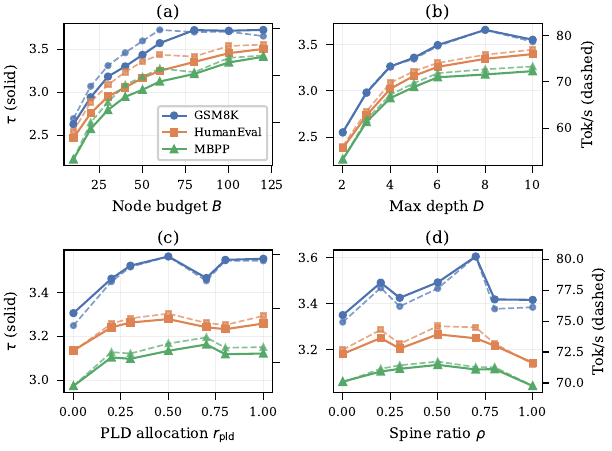}
\end{minipage}
\caption{\textbf{(a)}~\method{}'s mean~$\tau$ decomposed into the best standalone baseline $\max(\tau_{\pld{}},\tau_{\tr{}})$ (gray) and the synergy gain from the spine-tree topology (green); the gain ranges from $+16\%$ (Vicuna-13B) to $+45\%$ (Llama-3-8B).  \textbf{(b)}~Hyperparameter sensitivity (Qwen3-8B): each panel sweeps one parameter; $\tau$ plateaus at $B{\geq}60$ and $D{\geq}6$, and varies by fewer than 0.6~$\tau$ units across $r$ and~$\rho$. Panels (c)/(d) sweep the spine ratio $r$ and the spine-branch ratio $\rho$ of \Cref{sec:tree}.}
\label{fig:heterogeneity-decomp}
\vspace{-6pt}
\end{figure}

The synergy ratio $\tau_{\method{}}/\max(\tau_{\pld}, \tau_{\tr})$ exceeds~1.0 on every model (range 1.16--1.45$\times$, mean 1.24$\times$), confirming that the gain is a consistent structural phenomenon across all 25 settings.
The single highest cell is Llama-3-8B/GSM8K at 2.05, inflated by repetition (see \Cref{app:robustness}); excluding that cell, per-model synergy ranges from 1.16 to 1.24.

\subsection{Cross-Task Robustness}
\label{app:robustness}

\begin{figure}[h]
\centering
\begin{minipage}[t]{0.52\linewidth}
\centering
\includegraphics[width=\linewidth]{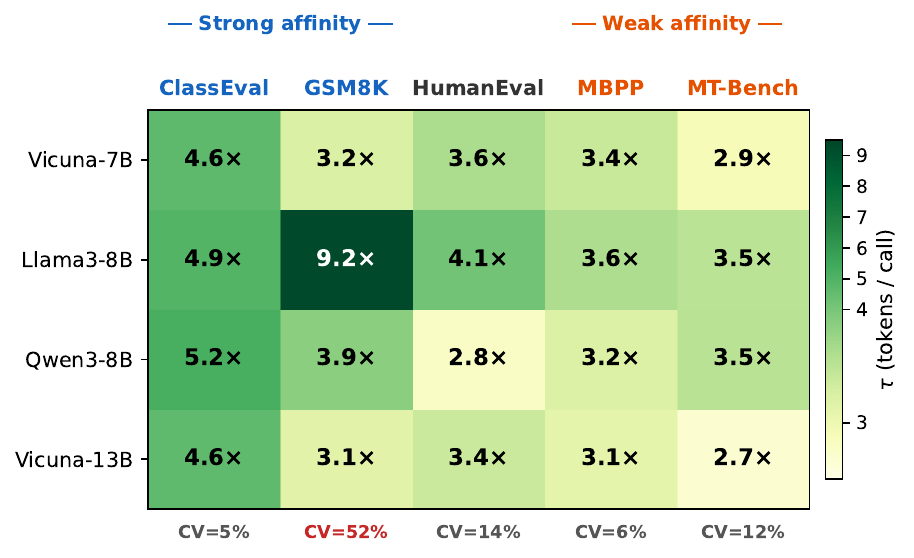}
\end{minipage}%
\hfill
\begin{minipage}[t]{0.46\linewidth}
\centering
\includegraphics[width=\linewidth]{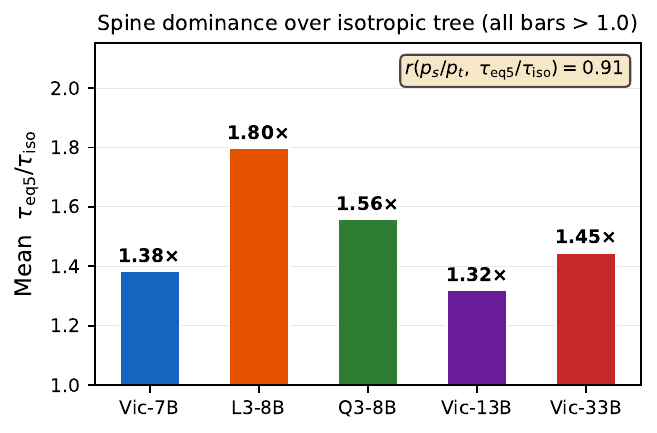}
\end{minipage}
\caption{\textbf{(a)}~$\taumetric$ across four A40 models and five datasets grouped by \method{}'s affinity. CV below each column shows cross-model coefficient of variation.
\textbf{(b)}~Per-model mean $\tau_{\text{eq5}}/\tau_{\text{iso}}$ across 24 settings; all bars exceed~1.0, confirming spine dominance (Pearson $r{=}0.91$).}
\label{fig:eq5-verify}
\end{figure}

Strong-affinity datasets (ClassEval, GSM8K) share repetitive structure, such as class templates or chain-of-thought boilerplate, which yields high \pld{}-path acceptance.
Weak-affinity datasets (MBPP, MT-Bench) contain shorter, less repetitive outputs that shift budget toward branches (most clearly on MBPP), yet still achieve $\taumetric > 2.6\times$.
The cross-model CV is below 14\% for four of five datasets: GSM8K is the exception (CV$=$52.2\%) because Llama-3-8B exhibits degenerate repetition loops on arithmetic tasks, inflating its $\taumetric$ to 9.25, an anomaly that disappears when that single model is excluded (CV drops to 9.5\%).
Overall, the strong/weak gap is moderate (mean 4.85 vs.\ 3.24, ratio 1.50$\times$; the strong-affinity mean is lifted by the Llama-3-8B GSM8K outlier noted above), showing that \method{} degrades gracefully rather than collapsing on unfavorable inputs; per-prompt variance within each cell is analyzed in \Cref{app:prompt-variance}.

% Hyperparameter Sensitivity moved to main text (Section 4 Analysis)

\subsection{Prompt-Level Variance}
\label{app:prompt-variance}

Across all 25 model$\times$dataset settings, median~$\tau$ tracks the mean closely (Pearson $r > 0.98$) and the interquartile range (IQR) stays within $0.5$--$2.3\times$ for 23 of 25 cases.
The two exceptions (Llama-3-8B/GSM8K and Qwen3-8B/ClassEval) are traced to degenerate repetition loops that inflate $\taumetric$; excluding the top~5\% of prompts reduces their IQR to levels consistent with other settings.
Overall, the tight IQR confirms that \method{}'s gains are broadly distributed across prompts rather than driven by outliers.

\subsection{EAGLE-2/3: Chat-Template and Raw-Prompt Protocols}
\label{app:eagle3-raw}

\begin{table}[h]
\centering\small
\caption{\method{} vs.\ EAGLE-3 under EAGLE's native chat-template protocol (Llama-3.1-8B, A40), the primary protocol for this baseline (\Cref{sec:eagle3}). $\tau$ / tok/s.}
\label{tab:eagle3-chat}
\begin{tabular}{@{}l cc cc@{}}
\toprule
& \multicolumn{2}{c}{\method{}} & \multicolumn{2}{c}{EAGLE-3} \\
\cmidrule(lr){2-3}\cmidrule(lr){4-5}
\textbf{Dataset} & $\tau$ & tok/s & $\tau$ & tok/s \\
\midrule
HumanEval & 2.59 & 66.6 & \textbf{5.79} & \textbf{109.0} \\
ClassEval & \textbf{5.35} & \textbf{131.3} & 5.34 & 94.5 \\
MT-Bench  & 2.22 & 57.0 & \textbf{5.36} & \textbf{102.0} \\
MBPP      & 2.36 & 61.4 & \textbf{5.69} & \textbf{108.7} \\
GSM8K     & 2.76 & 70.3 & \textbf{5.69} & \textbf{107.8} \\
\bottomrule
\end{tabular}
\end{table}

\begin{table}[h]
\centering\small
\caption{Raw-prompt protocol (our paper's protocol, applied to every baseline): Llama-3.1-8B, A40, full $n$. $\tau$ / tok/s.}
\label{tab:eagle3-raw}
\begin{tabular}{@{}l cc cc cc@{}}
\toprule
& \multicolumn{2}{c}{\method{}} & \multicolumn{2}{c}{EAGLE-2} & \multicolumn{2}{c}{EAGLE-3} \\
\cmidrule(lr){2-3}\cmidrule(lr){4-5}\cmidrule(lr){6-7}
\textbf{Dataset} & $\tau$ & tok/s & $\tau$ & tok/s & $\tau$ & tok/s \\
\midrule
HumanEval & \textbf{4.50} & \textbf{114.8} & 4.14 & 62.2 & 3.29 & 62.6 \\
ClassEval & \textbf{5.54} & \textbf{133.4} & 4.11 & 58.9 & 4.18 & 74.8 \\
MT-Bench  & 3.63 & \textbf{90.7} & 3.51 & 52.2 & \textbf{3.87} & 72.6 \\
MBPP      & 3.51 & \textbf{91.0} & \textbf{4.19} & 63.6 & 3.55 & 68.5 \\
GSM8K     & \textbf{8.90} & \textbf{221.7} & 4.28 & 63.6 & 4.62 & 86.2 \\
\bottomrule
\end{tabular}
\end{table}

On raw prompts (\Cref{tab:eagle3-raw}; chat protocol in \Cref{tab:eagle3-chat}) EAGLE's chat-trained head loses $\tau$, and its per-cycle draft forward erases the remaining lead: \method{} wins wall-clock on all five datasets, even where EAGLE's $\tau$ is comparable or higher (MT-Bench, MBPP).
GSM8K is fixed-budget (96\% truncated), so its $\tau$ and tok/s are inflated for the repetition-driven \method{} and should be read qualitatively.

\subsection{Per-Accept Source Attribution}
\label{app:source-attribution}

A natural concern is that when input $n$-gram matches are rare the spine degenerates.
We instrumented \method{} to label every accepted token by its source (\pld{} spine vs.\ \tr{} branch) on four models spanning 1.1B--8B.
\Cref{tab:source-attribution} reports \pld{}'s share of accepted tokens (\tr{} is the remainder).

\begin{table}[h]
\centering\small
\caption{\pld{} share of accepted tokens (\%); \tr{} supplies the remainder. Llama-3-8B and Qwen3-8B are paper models (full $n$); the two ${\sim}1$B models are robustness checks ($n{=}30$).}
\label{tab:source-attribution}
\begin{tabular}{@{}l rrrrr@{}}
\toprule
\textbf{Model} & ClassEval & MT-Bench & HumanEval & GSM8K & MBPP \\
\midrule
Qwen3-8B       & 70.8 & 53.3 & 33.0 & 48.3 & 38.3 \\
Llama-3-8B     & 62.1 & 53.5 & 53.4 & 83.8 & 45.2 \\
Llama-3.2-1B   & 61.4 & 35.1 & 60.7 & 37.9 & 36.7 \\
TinyLlama-1.1B & 66.1 & 84.5 & 33.2 & 59.8 & 33.4 \\
\bottomrule
\end{tabular}
\end{table}

\pld{}'s share is generally highest on template-heavy ClassEval (61--71\%) and on Llama-3-8B's repetition-prone GSM8K (84\%), and generally lowest on HumanEval/MBPP, where it falls to about a third.
Even at its sparsest, \tr{} supplies the remaining majority of accepts and $\tau$ does not collapse (e.g., $\tau{=}2.78$ on Qwen3-8B HumanEval, still $+12\%$ over the equal-budget isotropic tree, \Cref{tab:topology}): \tr{} needs no prompt match (it is built from the model's own logits), so the tree stays full when the spine is empty.
The spine therefore does not degenerate when $n$-gram matches are rare; the budget shifts to branches, which is precisely the adaptive behavior of \Cref{sec:consensus}.

\subsection{Swapping the Spine Source: PLD \texorpdfstring{$\to$}{->} SuffixDecoding}
\label{app:cross-source}

\pld{} and \tr{} are one instance of \method{}'s two acceptance classes, not a limit of the method: any source emitting a high-acceptance chain can serve as the spine, and any source of plausible alternatives as branches.
We show this by replacing the \pld{} spine with \emph{SuffixDecoding's suffix tree} (chain mode, official ArcticInference cache, not a re-implementation) while keeping the \tr{} branches and the topology unchanged.
The spine source is the only difference between the two arms; consensus bypass is off on both sides.
$\tau$ rises in all 15 model$\times$dataset cells ($+3\%$ to $+47\%$, \Cref{tab:cross-source}):

\begin{table}[h]
\centering\small
\caption{$\tau$ with the \pld{} spine swapped for SuffixDecoding's suffix tree (TR branches and topology unchanged, bypass off on both sides).}
\label{tab:cross-source}
\resizebox{\linewidth}{!}{
\begin{tabular}{@{}l ccccc@{}}
\toprule
$\tau$: \pld{}$\to$Suffix spine & HumanEval & MBPP & GSM8K & MT-Bench & ClassEval \\
\midrule
Llama-3.1-8B (full $n$) & 4.21$\to$\textbf{4.80} & 3.30$\to$\textbf{3.63} & 6.28$\to$\textbf{9.26} & 3.22$\to$\textbf{3.91} & 5.00$\to$\textbf{5.89} \\
Llama-3.2-1B ($n{=}30$) & 3.75$\to$\textbf{4.33} & 3.04$\to$\textbf{3.27} & 3.38$\to$\textbf{3.60} & 2.55$\to$\textbf{2.63} & 4.29$\to$\textbf{5.00} \\
TinyLlama-1.1B ($n{=}30$) & 2.84$\to$\textbf{3.08} & 2.88$\to$\textbf{3.09} & 4.56$\to$\textbf{5.41} & 6.19$\to$\textbf{8.92} & 5.03$\to$\textbf{5.74} \\
\bottomrule
\end{tabular}}
\end{table}

The suffix tree is a stronger chain source than \pld{} (empirical-frequency scoring vs.\ longest-exact match), yet the anisotropic topology and verification are unchanged: the gain composes with whichever spine source is strongest.
On Llama-3.1-8B the $\tau$ gain carries to wall-clock ($+7$--$42\%$ tok/s).
The swap also preserves losslessness: verification accepts only the target model's own argmax token regardless of draft source, and we confirmed token-identical output to greedy AR under fp32 for the swapped variant (Llama-3.2-1B HumanEval: 10/10 (method, sample) pairs across \{\method{}, \method{}-Suffix-spine\} $\times$ 5 prompts).

\subsection{Batch-Size Sweep}
\label{app:batch}

We sweep batch size with cross-request batched verification for \method{} against padded request-batched AR, on Llama-3-8B with the \emph{same} HuggingFace stack on both sides; a self-test verifies the batched decoder is batch-lossless (batched and bs${=}1$ decoding produce token-identical output, checked in fp32).
The \method{} prototype here is the bypass-off core decoder (budget 32), so its bs=1 speedup (${\sim}1.7$--$2.0\times$) understates the full system (3.36$\times$ on HumanEval); the object of interest is the \emph{shape} of the curves in \Cref{tab:batch-sweep}.

\begin{table}[h]
\centering\small
\caption{Aggregate tok/s (total tokens / wall-clock) vs.\ batch size, Llama-3-8B HumanEval ($n{=}164$) and GSM8K ($n{=}256$), greedy, fp16.}
\label{tab:batch-sweep}
\begin{tabular}{@{}ll rrrrr@{}}
\toprule
\textbf{GPU / task} & \textbf{Method} & bs=1 & bs=4 & bs=8 & bs=16 & bs=32 \\
\midrule
A40, HumanEval & \method{} & \textbf{59.0} & \textbf{129.8} & 141.5 & 149.4 & 143.0 \\
               & AR        & 31.0 & 96.5 & \textbf{168.2} & \textbf{265.0} & \textbf{366.1} \\
A40, GSM8K     & \method{} & \textbf{62.4} & \textbf{138.5} & \textbf{155.7} & 168.7 & 156.4 \\
               & AR        & 31.1 & 79.0 & 142.1 & \textbf{244.9} & \textbf{382.2} \\
RTX 3090, HumanEval & \method{} & \textbf{73.6} & 132.0 & 131.1 & 124.1 & --- \\
               & AR        & 42.4 & \textbf{136.1} & \textbf{228.3} & \textbf{339.0} & --- \\
RTX 3090, GSM8K & \method{} & \textbf{78.6} & \textbf{138.4} & 139.1 & 133.2 & --- \\
               & AR        & 42.2 & 111.7 & \textbf{197.9} & \textbf{327.3} & --- \\
\bottomrule
\end{tabular}
\end{table}

Speculative decoding and request batching draw on the same idle compute.
At bs=1, decoding is bandwidth-bound and \method{}'s tree verification accepts several tokens per weight load almost for free.
Batching amortizes those same weight loads across requests, so AR scales roughly linearly; \method{}, whose forward already processes batch~$\times$~tree-size tokens, saturates compute earlier and plateaus (on the A40, aggregate throughput rises ${\sim}2.6\times$ from bs=1 to the bs=16 plateau; on the 3090 it rises ${\sim}1.7\times$ and saturates by bs=4--8).
The crossover is dataset- and hardware-dependent but always present and always ordered the same way: GSM8K (longer generations, higher $\tau$) crosses later than HumanEval on both GPUs, and the A40 crosses later than the 3090.
GPU-utilization profiling during these runs is consistent with compute saturation rather than host overhead.
The prototype retains finished requests as padding rather than evicting them, wasting some large-batch compute; a production-grade batched implementation is left as future work.

\subsection{SuffixDecoding Mechanism Decomposition and Merged-Tree Variant}
\label{app:suffix-mech}

\Cref{tab:suffix-l318b} reports the head-to-head of \Cref{sec:suffix} on SuffixDecoding's own evaluation model.
Cell accounting for \Cref{sec:suffix}: the grid is 6 models (our five plus Llama-3.1-8B) $\times$ 13 datasets (our five plus SuffixDecoding's eight) $=78$ cells, all with $\tau$; ``full accounting'' means the 52 cells (four main-grid A40 models $\times$ 13 datasets; the models are Vicuna-7B, Llama-3-8B, Qwen3-8B, and Vicuna-13B) where wall-clock was logged under the paper harness, and the 32 own-dataset cells (those four models $\times$ its eight datasets) are a subset of the 52.

\begin{table}[h]
\centering\small
\caption{\method{} vs.\ SuffixDecoding (tree mode, the stronger of its two training-free modes) on Llama-3.1-8B (A40, greedy, full $n$). $\tau$ / tok/s. $^*$GSM8K fixed-budget.}
\label{tab:suffix-l318b}
\resizebox{\linewidth}{!}{
\begin{tabular}{@{}l ccccc@{}}
\toprule
& HumanEval & MBPP & MT-Bench & ClassEval & GSM8K$^*$ \\
\midrule
\method{}            & \textbf{4.50}/\textbf{114.8} & \textbf{3.51}/\textbf{91.0} & \textbf{3.63}/\textbf{90.7} & \textbf{5.54}/\textbf{133.4} & \textbf{8.90}/\textbf{221.7} \\
SuffixDecoding (tree) & 2.98/86.6 & 2.14/63.5 & 2.33/68.0 & 3.97/108.9 & 5.59/162.0 \\
\bottomrule
\end{tabular}}
\end{table}

On SuffixDecoding's own evaluation datasets (Llama-3.1-8B, $\tau$): SWE-Bench-Lite 4.04 vs.\ 2.78 ($+45\%$), Magicoder 2.78 vs.\ 1.67 ($+67\%$), WildChat 3.01 vs.\ 1.90 ($+58\%$), and $+20$--$64\%$ across the five Spec-Bench tasks.

\Cref{tab:suffix-merged} backs the mechanism claims of \Cref{sec:suffix} with the full 4-model $\times$ 5-dataset grid ($\tau$; $n{=}100$ ClassEval, $200$ GSM8K/MBPP, $164$ HumanEval, $80$ MT-Bench; max 512 new tokens; A40).
Arms: \method{} (cold, per-request, budget 60); SuffixDecoding official two-tree cache, cold and warm (cache seeded with the 8 preceding requests of the same dataset, sequential order); and our merged single-tree variant, warm.
The merged tree is \emph{our} construction (the published method keeps two trees and selects per cycle); it is built by per-sequence insertion and its semantics were machine-checked against the official code (35/35 assertions, including 300-trial fuzz bit-exactness).
In the main comparison (\Cref{sec:suffix}), all baselines run cold and per-request: SuffixDecoding's larger published agentic speedups use a warmed cache, a serving setting outside this benchmark rather than a capability we disabled, and its Hybrid modes are excluded as not training-free (they add a trained EAGLE-3 drafter).
Suffix arms use a max-spec factor of 2 (SuffixDecoding's per-match cap on draft length, $2\times$ the match length), which outperformed its default factor of 1 in all 20 cells, so the baseline is tuned generously.
Because of fp16 argmax near-ties, arms can diverge on a minority of samples (per-arm total token counts differ by up to 2.7\%).
We therefore compare $\tau$ (tokens per model call), computed on each arm's own trace.

\begin{table}[h]
\centering\small
\caption{$\tau$ on the full grid: \method{} (cold) vs.\ SuffixDecoding official (cold/warm) and the merged-tree variant (warm). \textbf{Bold} = column best per row.}
\label{tab:suffix-merged}
\begin{tabular}{@{}ll cccc@{}}
\toprule
\textbf{Model} & \textbf{Dataset} & \method{} & Suf.\ (cold) & Suf.\ (warm) & Merged (warm) \\
\midrule
Llama-3-8B   & ClassEval & \textbf{4.92} & 3.53 & 3.66 & 3.55 \\
             & GSM8K     & \textbf{5.90} & 3.38 & 3.79 & 3.68 \\
             & HumanEval & \textbf{4.13} & 2.76 & 2.94 & 2.88 \\
             & MBPP      & \textbf{3.58} & 2.25 & 2.54 & 2.54 \\
             & MT-Bench  & \textbf{2.71} & 1.63 & 1.73 & 1.71 \\
Llama-3.1-8B & ClassEval & \textbf{5.50} & 3.97 & 4.06 & 4.05 \\
             & GSM8K     & \textbf{4.99} & 3.13 & 3.36 & 3.24 \\
             & HumanEval & \textbf{4.53} & 2.98 & 3.20 & 3.15 \\
             & MBPP      & \textbf{3.44} & 2.12 & 2.41 & 2.39 \\
             & MT-Bench  & \textbf{2.69} & 1.67 & 1.73 & 1.71 \\
Llama-3.2-1B & ClassEval & \textbf{4.38} & 3.03 & 3.13 & 3.09 \\
             & GSM8K     & \textbf{2.97} & 1.77 & 2.06 & 2.03 \\
             & HumanEval & \textbf{4.05} & 2.68 & 2.99 & 2.95 \\
             & MBPP      & \textbf{3.25} & 1.99 & 2.29 & 2.25 \\
             & MT-Bench  & \textbf{2.39} & 1.48 & 1.55 & 1.54 \\
Qwen3-8B     & ClassEval & \textbf{3.50} & 2.22 & 2.55 & 2.55 \\
             & GSM8K     & \textbf{3.15} & 1.91 & 2.07 & 1.98 \\
             & HumanEval & \textbf{2.74} & 1.69 & 1.90 & 1.88 \\
             & MBPP      & \textbf{3.13} & 1.96 & 2.11 & 2.08 \\
             & MT-Bench  & \textbf{2.79} & 1.71 & 1.73 & 1.74 \\
\bottomrule
\end{tabular}
\end{table}

Three readings.
(i)~\method{} cold beats every suffix arm, including warm, in all 20 cells ($1.35$--$1.60\times$ the best suffix $\tau$), despite the protocol asymmetry favoring SuffixDecoding.
(ii)~The merged tree trails the official two-tree method in 18 of 20 cells (typically $-0.5$ to $-4.5\%$; the two exceptions are ${\le}0.6\%$ wins): pooling prompt-local and global statistics dilutes locally-certain continuations below the drafting threshold, the dilution mechanism we anticipated when constructing the variant.
(iii)~The per-cycle instrumentation (same runs) locates \method{}'s advantage: dry-cycle fraction (cycles whose draft is empty or a single token) 0.11--0.22 vs.\ 0.22--0.45; nodes drafted per usable (non-dry) cycle 52--60 at depth 5.7--22 vs.\ 3.4--7.9 at depth 2.9--7.3; and first-quartile-of-generation accepted tokens per cycle 0.92--2.90 vs.\ 0.25--1.65.

\subsection{Sequoia: Official Head-to-Head and Modern-Stack Full Grid}
\label{app:sequoia-modern}

\Cref{tab:sequoia-headtohead} gives the head-to-head against Sequoia's official released stack.
Both systems run under their own released protocols: Sequoia prefills 128 tokens and decodes to a fixed 256-token total (its official setting), while \method{} generates longer (220--256 tokens per sample).
Prefill is charged once per sample, so it is amortized over ${\approx}2\times$ more cycles for \method{}, inflating its tok/s by roughly two points relative to a length-matched run; $\tau$ is essentially unaffected (\method{}'s $\tau$ counts the prefill forward while Sequoia's official counters exclude it, which only understates \method{}).

\begin{table}[h]
\centering\small
\caption{\method{} vs.\ official Sequoia on Llama-2-7B (A40, greedy, $n{=}50$; \Cref{sec:sequoia-headtohead}). Sequoia uses its released trained 68M drafter; \method{} is training-free. Wall-clock margins carry a ${\sim}2$-point protocol uncertainty from the generation-length asymmetry described above.}
\label{tab:sequoia-headtohead}
\begin{tabular}{@{}l cc cc c@{}}
\toprule
& \multicolumn{2}{c}{\method{}} & \multicolumn{2}{c}{Sequoia (official)} & \\
\cmidrule(lr){2-3}\cmidrule(lr){4-5}
\textbf{Dataset} & $\tau$ & tok/s & $\tau$ & tok/s & Sequoia paper $\tau$ \\
\midrule
C4           & 4.16 & \textbf{116.9} & \textbf{4.79} & 107.8 & 5.08 \\
CNN/DM       & 3.04 & 86.6 & \textbf{3.91} & \textbf{91.1} & 4.05 \\
OpenWebText  & 3.28 & \textbf{93.4} & \textbf{3.95} & 91.9 & 3.86 \\
\bottomrule
\end{tabular}
\end{table}

\Cref{tab:sequoia-modern} reports our positional-DP reimplementation of Sequoia on modern targets (\Cref{sec:sequoia-headtohead}), with the drafter forward charged into wall-clock; a reimplementation is required because the released stack hard-codes Llama-2-era rotary embeddings and vocabulary sizes and supports neither Llama-3.1's RoPE scaling nor Qwen3's QK-norm.
$\tau$ follows Sequoia's counting convention (generated tokens per drafting--verification cycle, prefill excluded).
Sequoia's acceptance vector is calibrated on 8 held-back samples per dataset; on the three full-pool datasets this slice overlaps the evaluation set, an overlap that can only favor Sequoia.
Validation proceeds in two steps.
First, re-running Sequoia's released tree-search on its released Llama-2-7B~+~68M acceptance-rate vector reproduces our DP's tree level-for-level, with identical level sizes and expected accepted length.
Second, the measured $\tau$ of our runner is nevertheless $8.6$--$24.2\%$ \emph{above} the official code on the same data, a residual we did not trace to a single cause, as the stacks differ beyond the tree itself.
Since the gap only inflates Sequoia's numbers, the wall-clock conclusions below are conservative.
Lossless-rate accounting on the modern-stack cells (most severely Qwen3) is confounded by fp16 argmax near-ties (as in \Cref{app:suffix-mech}), so we report $\tau$ and wall-clock only.

\begin{table}[h]
\centering\small
\caption{Modern-stack Sequoia: $\tau$ and wall-clock speedup vs.\ AR at tree sizes $N{=}60/128$ ($N$ = node count in Sequoia's notation; $N{=}60$ matches our budget $B$; A40, greedy, full $n$). Drafters: Llama-3.2-1B for Llama targets, Qwen3-0.6B for Qwen3-8B.}
\label{tab:sequoia-modern}
\begin{tabular}{@{}ll cc cc@{}}
\toprule
& & \multicolumn{2}{c}{$N{=}60$} & \multicolumn{2}{c}{$N{=}128$} \\
\cmidrule(lr){3-4}\cmidrule(lr){5-6}
\textbf{Target} & \textbf{Dataset} & $\tau$ & Spd & $\tau$ & Spd \\
\midrule
Llama-3-8B   & ClassEval & 8.63 & 1.52 & 8.95 & 1.52 \\
             & GSM8K     & 8.89 & 1.56 & 9.12 & 1.54 \\
             & HumanEval & 8.56 & 1.53 & 8.88 & 1.51 \\
             & MBPP      & 8.42 & 1.48 & 8.75 & 1.48 \\
             & MT-Bench  & 7.57 & 1.36 & 8.00 & 1.38 \\
Llama-3.1-8B & ClassEval & 8.98 & 1.58 & 9.24 & 1.57 \\
             & GSM8K     & 8.97 & 1.60 & 9.15 & 1.56 \\
             & HumanEval & 8.72 & 1.54 & 8.95 & 1.52 \\
             & MBPP      & 8.35 & 1.48 & 8.66 & 1.47 \\
             & MT-Bench  & 7.81 & 1.37 & 8.30 & 1.40 \\
Qwen3-8B     & ClassEval & 7.32 & 0.82 & 7.81 & 0.85 \\
             & GSM8K     & 7.58 & 0.85 & 8.17 & 0.89 \\
             & HumanEval & 6.81 & 0.76 & 7.47 & 0.82 \\
             & MBPP      & 7.34 & 0.82 & 7.88 & 0.86 \\
             & MT-Bench  & 6.22 & 0.70 & 6.71 & 0.73 \\
\bottomrule
\end{tabular}
\end{table}

The high-$\tau$/low-speedup pattern is structural: Sequoia's per-cycle drafter forwards (a 1B/0.6B model on a 60--128-node tree) cost a large fraction of what verification saves, and on Qwen3-8B they cost \emph{more}: every cell is below $1\times$.
\method{} needs no drafter forward, which is exactly the regime difference \Cref{sec:training-free-value} describes.

\subsection{Principled Dynamic-Tree Allocator: Full Grid}
\label{app:allocator-full}
\label{sec:allocator}

\Cref{prop:allocation} prescribes a linear branch schedule, which \method{} approximates with a hand-tuned harmonic rule (\Cref{sec:tree}).
To measure what that approximation costs, we replace it with the greedy dynamic-tree expansion used by trained-drafter methods~\citep{li2024eagle2}, parameterized by \emph{measured} acceptance rather than drafter confidence: expand the global top-$B$ candidates under $\mathrm{Score}(\text{child}) = \mathrm{Score}(\text{parent}) \cdot \hat{p}[\text{source}][\text{rank}]$, where $\hat{p}$ is the empirical acceptance frequency of a source's rank-$k$ candidate (rank = position among a node's candidates by draft score), calibrated on a held-out prompt set disjoint from all evaluation benchmarks; multiplied along a path, the per-rank table induces the per-depth acceptance profile.
Under independence, $\mathrm{Score}(v)$ is the probability that $v$'s entire root path is accepted, so expected accepted tokens $=$ summed scores; no child outscores its parent, so the top-$B$ set forms a tree and greedy expansion is exactly optimal for this objective.
At strict equal budget ($B{=}60$) its static arm beats the deployed heuristic on 17 of 19 full-size cells (geometric mean $+5.9\%$; \Cref{tab:allocator-full}).
The two losing cells, GSM8K on the 8B Llamas, share one cause: long verbatim repetitions push the realized spine acceptance far above its trace average, so the static $\hat{p}$ under-builds the deep spine.
The deployed heuristic can thus be read as a hand-tuned approximation that the principled allocator subsumes.

\Cref{tab:allocator-full} reports the full equal-budget grid ($\tau$; full $n$ except GSM8K at $n{=}500$; $B{=}60$).
All arms spend an identical budget: the deployed heuristic runs with its confidence budget-extension disabled, the mechanism that would otherwise let high-confidence cycles slightly exceed $B$.
``Inst.''\ scores each candidate by a calibrated map from its own instance features (adjacency count, match length) to acceptance, the closest training-free analogue of a trained drafter's confidence.

\begin{table}[h]
\centering\small
\caption{Equal-budget allocation comparison, $\tau$ (full $n$ except GSM8K at $n{=}500$). \textbf{Bold} = row best. The two GSM8K/8B-Llama regressions are diagnosed above.}
\label{tab:allocator-full}
\begin{tabular}{@{}ll ccc@{}}
\toprule
\textbf{Model} & \textbf{Dataset} & Heuristic & Allocator & Allocator (inst.) \\
\midrule
Llama-3-8B   & ClassEval & 4.80 & \textbf{5.39} & 5.24 \\
             & GSM8K     & \textbf{8.84} & 7.48 & 7.25 \\
             & HumanEval & 4.03 & 4.47 & \textbf{4.50} \\
             & MBPP      & 3.58 & 3.92 & \textbf{3.98} \\
             & MT-Bench  & 3.41 & \textbf{3.56} & 3.47 \\
Llama-3.1-8B & ClassEval & 5.46 & \textbf{5.95} & 5.81 \\
             & GSM8K     & \textbf{8.15} & 7.38 & 7.18 \\
             & HumanEval & 4.47 & 4.86 & \textbf{4.96} \\
             & MBPP      & 3.48 & 3.78 & \textbf{3.85} \\
             & MT-Bench  & 3.62 & 3.69 & \textbf{3.71} \\
Llama-3.2-1B & ClassEval & 4.27 & 4.59 & \textbf{4.78} \\
             & GSM8K     & 3.17 & 3.48 & \textbf{3.57} \\
             & HumanEval & 3.98 & 4.28 & \textbf{4.40} \\
             & MBPP      & 3.18 & 3.50 & \textbf{3.58} \\
             & MT-Bench  & 2.68 & 2.85 & \textbf{2.89} \\
Qwen3-8B     & ClassEval & 3.44 & \textbf{3.94} & 3.92 \\
             & HumanEval & 2.73 & 2.96 & \textbf{2.96} \\
             & MBPP      & 3.14 & 3.45 & \textbf{3.47} \\
             & MT-Bench  & 3.43 & 3.53 & \textbf{3.54} \\
\bottomrule
\end{tabular}
\end{table}

Qwen3-8B GSM8K is omitted because its instance-calibrated arm did not complete; on the two arms that did run ($n{=}500$), the allocator leads the heuristic (allocator $4.10$ vs.\ heuristic $3.86$).

\clearpage
\section{Theoretical Proofs and Verification}
\label{app:proofs}

\subsection{Proof of Proposition~\ref{prop:yield} (Spine Tree Expected Yield)}
\label{app:yield-proof}

\begin{proof}
We use the indicator random variable (IRV) decomposition introduced for speculative decoding by \citet{leviathan2023fast}.
For each position in the speculation tree, let $I = \mathbf{1}[\text{the accepted path reaches and accepts this position}]$.
Since $\tau = \sum_{\mathrm{positions}} I$ (plus the bonus token), linearity of expectation gives
$\mathbb{E}[\tau] = \sum \Pr[I=1] + 1$.

\textbf{Spine positions ($d=1,\ldots,m$).}
Spine position~$d$ is reached and accepted iff all $d$ spine tokens along the chain are accepted.
Under the i.i.d.\ acceptance model, $\Pr[I=1] = p_s^d$.
Summing: $\mathbb{E}[\text{spine}] = \sum_{d=1}^{m} p_s^d$.

\textbf{Branch positions at node $i$ ($i=0,\ldots,m{-}1$).}
A branch subtree at node~$i$ is entered iff (a) all $i$ preceding spine tokens are accepted (probability~$p_s^i$) and (b) the $(i{+}1)$-th spine token is rejected (probability~$1{-}p_s$).
Given entry, the branch token (depth offset $j=0$) is accepted with probability $\phi_i = 1{-}(1{-}p_t)^{w_i}$.
For $j=1,\ldots,D{-}1$, each subsequent chain token is accepted independently with probability~$p_t$, so depth offset~$j$ is \emph{reached and accepted} with probability $\phi_i \cdot p_t^j$.
Summing over all offsets within the branch subtree:
\[
  \mathbb{E}[\text{branch at node }i \mid \text{entry}]
    = \sum_{j=0}^{D-1} \phi_i\, p_t^j
    = \phi_i\!\left(1 + \sum_{j=1}^{D-1}p_t^j\right)
    = \phi_i(1+\bar\ell),
\]
where $\bar\ell = \sum_{j=1}^{D-1}p_t^j$ as defined in the main text.
Here the ``$1$'' accounts for the branch token itself, and $\bar\ell$ for the expected subsequent chain extension.
Multiplying by the entry probability and summing over all spine nodes:
\[
  \mathbb{E}[\text{synergy}] = \sum_{i=0}^{m-1} p_s^i(1{-}p_s)\,\phi_i\,(1{+}\bar\ell).
\]

\textbf{Combining.}
Adding the spine contribution, the synergy contribution, and the $+1$ bonus token:
\[
  \mathbb{E}[\tau] \;\geq\; \sum_{i=1}^{m} p_s^i \;+\; \sum_{i=0}^{m-1} p_s^i(1{-}p_s)\,\phi_i\,(1{+}\bar\ell) \;+\; 1.
\]
The inequality holds because branches are modeled as independent linear chains; any deeper sub-branching within a branch (absent in this setup) would only add further contributions, confirming the lower-bound direction.
The bound is tight when each branch subtree is an independent linear chain with no further sub-branching.
\end{proof}

\subsection{Proof of Optimal Branch Allocation}
\label{app:allocation-proof}

\begin{proof}
Let $\Delta = \sum_{i=0}^{m-1} p_s^{i}(1{-}p_s)\,\phi_i\,(1{+}\bar\ell)$ denote the synergy term in \Cref{eq:yield}, with $\phi_i = 1-(1-p_t)^{w_i}$.
The marginal synergy of adding one branch at position~$i$ is:
\[
\frac{\partial \Delta}{\partial w_i} = p_s^i(1{-}p_s)(1{-}p_t)^{w_i}\,\ln\!\tfrac{1}{1{-}p_t}\,(1{+}\bar\ell).
\]
At an interior optimum under the constraint $\sum w_i = B_t$, all marginals must be equal (zero-width coordinates instead obey the Karush--Kuhn--Tucker (KKT) inequality, their marginal not exceeding the common value):
$p_s^i(1{-}p_t)^{w_i} = \mu$ for a Lagrange constant~$\mu$.
Taking logarithms:
$i \ln p_s + w_i \ln(1{-}p_t) = \ln \mu$,
which gives $w_i = w_0 - \frac{|\ln p_s|}{|\ln(1{-}p_t)|} \cdot i$, i.e., linearly decreasing with depth.
The intercept $w_0$ is set by $\sum_{i=0}^{m-1} w_i = B_t$.
\end{proof}

\subsection{Proof of Anisotropic Dominance}
\label{app:dominance-proof}

\begin{proof}
Let the total node budget be $B$ and define $p_s > p_t > 0$.

\textbf{Step 1 (Replacement lemma).}
Consider any tree $\mathcal{T}$ of budget~$B$ where every candidate token has acceptance probability~$p_t$.
At depth~$d$, let $w_d$ denote the number of candidates; the probability of advancing past depth~$d$ is $q_d = 1-(1-p_t)^{w_d}$.
Replace one of the $w_d$ candidates with a spine token of acceptance~$p_s$:
\[
  q_d' \;=\; 1-(1-p_s)(1-p_t)^{w_d-1}
       \;=\; q_d \;+\; (p_s - p_t)\,(1-p_t)^{w_d-1}.
\]
Since $p_s > p_t$ and $(1-p_t)^{w_d-1} > 0$, we have $q_d' > q_d$.
The expected yield of any tree is $Y = \sum_{d=1}^{D}\prod_{j=1}^{d} q_j + 1$, which is strictly increasing in each~$q_j$.
Hence $Y' > Y$: replacing one branch token with a spine token at any single depth strictly improves yield.

\textbf{Step 2 (Constructing a dominating spine tree).}
Start from the optimal isotropic tree~$\mathcal{T}_{\text{iso}}^*$ of budget~$B$---the optimum of Sequoia-style DP~\citep{chen2024sequoia} restricted to a single source at rate~$p_t$---which maximizes $Y$ over all such trees.
Let $D$ be its maximum depth.
For $d = 1, \ldots, \min(m, D)$, apply Step~1 along a single root-to-leaf path, replacing one $p_t$-token per depth with a $p_s$-spine token.
The budget is unchanged ($B$), but each replaced depth has $q_d' > q_d$, so:
\[
  Y_{\text{spine}} \;>\; Y_{\text{iso}}^*.
\]
This spine tree is a feasible (not necessarily optimal) instantiation of \Cref{prop:yield}.

\textbf{Step 3 (Optimization improves further).}
Reallocating branch budget according to \Cref{prop:allocation} can only increase the yield:
$Y_{\text{spine}}^* \geq Y_{\text{spine}} > Y_{\text{iso}}^*$.

\textbf{Step 4 (Gap monotonicity).}
The per-depth gain $(p_s{-}p_t)(1{-}p_t)^{w_d-1}$ increases linearly in $(p_s - p_t)$, so $Y_{\text{spine}}^* - Y_{\text{iso}}^*$ is monotonically increasing in $p_s/p_t$ (with $p_t$ fixed) and vanishes as $p_s \to p_t$.
\end{proof}

\subsection{Numerical Verification of the Yield Bound}
\label{app:eq5-verify}

We verify the yield lower bound of \Cref{prop:yield} against measured compression ratios across all five target models and five benchmarks (24 settings; the Vicuna-33B GSM8K cell is excluded because its diagnostics come from a separate run of that environment; cf.\ \Cref{tab:main-full}).
For each setting we extract $\hat{p}_s$ and $\hat{p}_t$ from the logged spine and branch statistics, estimate the average tree shape $(m, B)$ from the per-cycle tree-shape logs, and compute three quantities:
(i)~$\tau_{\text{eq5}}$, the lower bound from Eq.~\ref{eq:yield} with the optimal linear allocation of \Cref{prop:allocation};
(ii)~$\tau_{\text{iso}}$, the yield of an isotropic tree with fan-out~3 and the same budget using only~$p_t$;
(iii)~$\tau_{\text{meas}}$, the measured per-prompt mean of accepted tokens per cycle (a per-prompt convention that differs from the pooled $\tau$ of \Cref{tab:main-full}).

The bound is valid ($\tau_{\text{eq5}} \leq \tau_{\text{meas}}$) in all 24 model$\times$dataset settings, confirming that \Cref{prop:yield} provides a correct lower bound.
The bound is conservative (median $\tau_{\text{eq5}}/\tau_{\text{meas}} \approx 38\%$), which is expected: the theory does not account for the consensus bypass mechanism (\Cref{sec:consensus}) or multi-cycle warm-start effects, both of which raise the effective acceptance rate in practice.

More importantly, $\tau_{\text{eq5}} > \tau_{\text{iso}}$ in all 24 settings (median ratio $\approx 1.52\times$), empirically confirming the dominance result of \Cref{prop:dominance}.
The gain correlates strongly with the heterogeneity ratio $\hat{p}_s/\hat{p}_t$ (Pearson $r = 0.91$), consistent with the monotonicity predicted by Step~4 of the proof (\Cref{app:dominance-proof}).
\Cref{fig:eq5-verify} visualizes the per-model dominance ratio.

\subsection{Empirical Tail of the All-Reject Probability}
\label{app:r_k-empirical}

For every node of the speculation tree we log the branching factor $k$ and
whether all $k$ candidate children were rejected, over 10 trace files
(Llama-3-8B and Llama-3.2-1B across the five benchmarks; 93{,}791 nodes).
Within each $k$-bin we estimate the all-reject frequency
$r_k = \Pr[\text{all }k\text{ candidates rejected}]$ and fit two families by least squares in log space ($R^2$ reported on $\log r_k$): geometric, $r_k = c\,q^{k}$, and power law,
$r_k = c\,k^{b}$.
Because Goose's anisotropic tree induces a structural correlation between
$k$ and depth (wide forks sit near the root), we fit both on the aggregate
data ($k$-bins with $\ge 200$ observations) and within fixed-depth bins
($\ge 100$ observations per cell):
\begin{center}\small
\begin{tabular}{lcccc}
\toprule
Bin & $k$-bins & obs. & Geometric $R^2$ & Power-law $R^2$ \\
\midrule
aggregate   & 11 & 93{,}791 & 0.387 & \textbf{0.606} \\
depth $0$   & 11 & 28{,}787 & 0.466 & \textbf{0.587} \\
depth $1$   & 10 & 19{,}843 & 0.692 & \textbf{0.722} \\
\bottomrule
\end{tabular}
\end{center}
Deeper bins are sparse (3--7 populated $k$-cells) and neither family fits
well there (at depths $2$--$5$, $R^2\le 0.27$ for both, with the geometric
nominally ahead; depth $6$ has only $3$ cells).
On the well-populated bins the power law fits better, with exponent magnitude $0.261$ (aggregate) and $0.189$ (depth $0$). The raw fitted slopes are \emph{positive} ($r_k$ grows with $k$), a selection effect of the adaptive builder discussed in \Cref{app:powerlaw-allocation}; the magnitudes serve only as tail-heaviness scales in the causal, decreasing orientation.

\subsection{Optimal Branch Allocation under a Power-Law Tail}
\label{app:powerlaw-allocation}

\paragraph{Setup.}
\Cref{prop:allocation} derives the optimal branch widths
$\{w_i^\ast\}$ under the i.i.d.\ model
$\Pr[\text{all reject at depth }i\mid w_i] = (1-p_t)^{w_i}$,
i.e.\ a \emph{geometric} tail in $w_i$.
The i.i.d.\ assumption is a modeling approximation; here we ask how the
optimal allocation changes when it is replaced by the tail family that the
trace data actually prefer.
Empirically (\Cref{app:r_k-empirical}), the all-reject frequency $r_k$ as a
function of the number of branch candidates $k$ is better described by a
power law than by a geometric: the aggregate fit over $11$ $k$-bins gives
$R^2=0.606$ (power law) versus $0.387$ (geometric) with fitted exponent
magnitude $\widehat{\alpha}=0.261$, and the same ordering holds within the
most populous fixed-depth bin, which controls for the structural
$k\!\leftrightarrow\!\text{depth}$ correlation of the anisotropic tree
(depth~$0$, $n{=}28{,}787$: $R^2=0.587$ versus $0.466$, exponent
magnitude $0.189$).
We use the aggregate $\widehat{\alpha}=0.261$ below and report robustness to
the depth-$0$ value.
One caveat is stated up front: the \emph{raw} fitted slope is positive
($r_k$ grows with $k$) because the tree builder adaptively assigns more
candidates at harder positions (a selection effect, not a causal effect of
$k$), so we use the fitted magnitude only as a measure of tail heaviness,
in the causal (decreasing) orientation of \eqref{eq:powerlaw-tail}.

\paragraph{Model.}
At spine depth $i$, replace the geometric per-depth synergy factor with a power-law
objective term $\tilde\phi_i$ (a weight absorbing the reach probability, not itself a probability; $q_i$ here is distinct from $q_d$ in \Cref{app:dominance-proof}):
\begin{equation}
\label{eq:powerlaw-tail}
\tilde\phi_i(w_i) \;=\; 1 \;-\; q_i\,w_i^{-\alpha},
\qquad
q_i \;=\; (1-p_t)\,p_s^{\,i},
\end{equation}
where $p_s^{\,i}$ carries the probability that the spine reaches depth $i$
(mirroring the synergy weights of \Cref{prop:yield}; the geometric
benchmark below uses the same reach weights, so the comparison is
like-for-like) and the factor $(1-p_t)$ calibrates the tail to the i.i.d.\
model at $w_i=1$.
For large $w$ this tail is heavier than geometric (at the fitted parameters the two curves cross near $w{\approx}26$); the geometric tail is recovered in the $\alpha\to\infty$ limit of the shifted power-law
(Lomax) family, $\bigl(1+w/(\alpha\sigma)\bigr)^{-\alpha}\to e^{-w/\sigma}$,
of which \eqref{eq:powerlaw-tail} is the heavy-tailed end.

\paragraph{Lagrangian.}
Maximising $\sum_i \tilde\phi_i(w_i)$ under the branch-budget constraint
$\sum_i w_i = B_t$ gives
\begin{equation*}
\mathcal{L}\;=\;\sum_i \bigl[1 - q_i\,w_i^{-\alpha}\bigr]
            \;-\;\lambda\!\left(\textstyle\sum_i w_i - B_t\right).
\end{equation*}
Stationarity $\partial\mathcal{L}/\partial w_i = 0$ yields
$\alpha\,q_i\,w_i^{-\alpha-1} = \lambda$,
hence
\begin{equation}
\label{eq:powerlaw-optimum}
\boxed{\,w_i^{\ast}
\;=\;
C\cdot p_s^{\,i/(\alpha+1)},
\qquad
C\;\text{set by } \sum_{i=0}^{m-1} w_i^{\ast} = B_t.\,}
\end{equation}
Each summand is strictly concave on $w_i>0$ and the marginal value
$\alpha q_i w_i^{-\alpha-1}$ diverges as $w_i\to 0^{+}$, so
\eqref{eq:powerlaw-optimum} is the unique constrained optimum and, unlike the geometric case below, every depth receives strictly positive width. Widths below one candidate lie outside the fitted region ($k\ge1$), so positivity there is an artifact of the continuous relaxation; the substantive contrast is the graded shallow-depth profile versus the corner.

\paragraph{Effect of the tail model on the optimum.}
Under the geometric model, the stationarity condition gives the linear
profile of \Cref{prop:allocation},
$w_i = w_0 - i\,|\ln p_s|/|\ln(1{-}p_t)|$; making the non-negativity
constraint explicit, the exact optimum (by the KKT conditions) is its water-filling truncation
$w_i^{\ast} = \max\bigl(0,\;w_0 - i\,|\ln p_s|/|\ln(1{-}p_t)|\bigr)$.
At the measured medians ($p_s{=}0.21$, $p_t{=}0.033$) the slope is
$|\ln 0.21|/|\ln 0.967| \approx 46.5$ (larger than the entire branch budget
$B_t{=}30$), so the geometric optimum degenerates to the \emph{corner}
solution: the whole budget at the root, $w^{\ast} = (B_t,0,\dots,0)$. The corner obtains for any $B_t \le 46.5$; under \Cref{prop:allocation}'s accounting $B_t{=}B{-}m{=}48$, a second depth becomes active with width $0.75$, leaving the contrast unchanged.
Under the power-law tail, the optimum \eqref{eq:powerlaw-optimum} instead
decays \emph{geometrically} in $i$ at rate
$p_s^{1/(\alpha+1)} = 0.21^{1/1.261} \approx 0.29$
(widths shrink ${\approx}3.4\times$ per depth step) and keeps a strictly
positive width at every depth.
Replacing the geometric assumption with the empirically fitted tail thus
changes the optimum from ``all branches at the root'' to a graded,
root-heavy schedule with a long shallow tail; the qualitative anisotropy of
\Cref{prop:allocation} (concentrate branches near the root) is preserved
under both models, and the heavier-than-geometric empirical tail is exactly
what rewards retaining residual breadth at deeper nodes.
(With the depth-$0$ exponent $0.189$ the decay rate is
$0.21^{1/1.189}\approx 0.27$; the schedule is essentially unchanged.)

\paragraph{Position of the deployed schedule.}
With branch budget $B_t{=}30$ (the branch share at the spine-ratio ceiling $r{=}0.5$; the mid tier $r{=}0.30$ gives $B_t{\approx}42$, also below the corner threshold), $p_s{=}0.21$, $p_t{=}0.033$,
and the schedule truncated at $m{=}12$ depths (a representative spine length; the comparison is qualitatively insensitive to $m$),
the three schedules are
\begin{center}\small
\begin{tabular}{lcccccccc}
\toprule
$i$                   & 0 & 1 & 2 & 3 & 4 & 5 & 6 & $\ge 7$ \\
\midrule
$w_i^\ast$ geometric (KKT) & 30.0 & 0.0 & 0.0 & 0.0 & 0.0 & 0.0 & 0.0 & 0.0 \\
$w_i^\ast$ power-law   & 21.3 & 6.2 & 1.8 & 0.5 & 0.2 & 0.0 & 0.0 & 0.0 \\
$1/(i{+}1)$ (deployed rule, idealized)   & 9.7 & 4.8 & 3.2 & 2.4 & 1.9 & 1.6 & 1.4 & 1.2--0.8\\
\bottomrule
\end{tabular}
\end{center}
The deployed harmonic schedule is \emph{more conservative} (less
root-concentrated) than both optima at the root
($w_0{=}9.7$ vs.\ power-law $21.3$ and geometric $30$) and retains breadth
at deep nodes to which neither optimum assigns appreciable width.
Measured in $\ell_1$ over normalised schedules (each scaled to unit sum), the deployed heuristic is
closer to the power-law optimum ($\ell_1{=}0.87$) than to the geometric one
($\ell_1{=}1.36$): moving from the geometric assumption to the empirically
fitted tail moves the predicted optimum \emph{toward} the deployed schedule.
Goose's hand-tuned $1/(i{+}1)$ rule is thus a deliberately conservative
compromise (we idealize it here by its harmonic schedule; Algorithm~1 additionally gives the root a separate share and indexes the spine harmonic from depth~1, which together raise the deployed root share to ${\approx}0.49$ of the branch budget without changing the ordering): it implements the monotone, root-heavy shape on which both
models agree, while hedging against model misspecification with residual
breadth at depths where the power-law (but not the geometric) model still
assigns positive width.
We expect a non-degradation guarantee analogous to \Cref{prop:dominance} to carry
over under either model, since the largest width satisfies
$w_0^{\ast}\ge B_t/m \ge 1$ whenever $B_t\ge m$.

\paragraph{Scope of the analysis.}
The allocation analysis governs only \emph{how} a fixed branch budget is
spread, not \emph{whether} heterogeneous allocation helps: the 12--33\%
equal-budget gains of \Cref{tab:topology} are measured end-to-end and hold
regardless of which tail model is assumed.
Both the geometric and the power-law optima prescribe the same operative
shape, monotone root-heavy widths, which is what the deployed
$1/(i{+}1)$ rule implements; the empirical power-law fit additionally
explains why the heuristic's residual depth-side breadth is not wasteful.

\end{document}